\documentclass[sigconf]{acmart}

\copyrightyear{2026}
\acmYear{2026}
\setcopyright{cc}
\setcctype{by}
\acmConference[WWW '26]{Proceedings of the ACM Web Conference 2026}{April 13--17, 2026}{Dubai, United Arab Emirates}
\acmBooktitle{Proceedings of the ACM Web Conference 2026 (WWW '26), April 13--17, 2026, Dubai, United Arab Emirates}
\acmPrice{}
\acmDOI{10.1145/3774904.3792244}
\acmISBN{979-8-4007-2307-0/2026/04}

\usepackage{hyperref}
\usepackage{subfig}
\usepackage{url}
\usepackage{tikz}
\usepackage{pgf-umlsd}
\usepackage{multirow}
\usepackage{tcolorbox}
\usetikzlibrary{positioning}
\usepackage{arydshln}
\usepackage[shortlabels]{enumitem} 
\usepackage{appendix} 
\usepackage[table,xcdraw]{xcolor}
\usepackage{graphicx}
\usepackage{xcolor}
\usepackage{booktabs}
\usepackage{xspace}
\usepackage{natbib}
\usepackage{makecell}
\usepackage{enumitem}

\theoremstyle{plain}

\usepackage[ruled,linesnumbered]{algorithm2e}

\newtheorem{lemma}{Lemma}[section]
\newtheorem{theorem}{Theorem}[section]

\begin{document}

\title{Traceable Latent Variable Discovery Based on \mbox{Multi-Agent Collaboration}}

\author{Huaming Du}
\affiliation{%
  \institution{Southwestern University of Finance and Economics}
  \city{Chengdu}
  \state{Sichuan}
  \country{China}
}

\author{Tao Hu}
\affiliation{%
  \institution{Southwestern University of Finance and Economics}
  \city{Chengdu}
  \state{Sichuan}
  \country{China}}

\author{Yijie Huang}
\affiliation{%
  \institution{Southwestern University of Finance and Economics}
  \city{Chengdu}
  \state{Sichuan}
  \country{China}}

\author{Yu Zhao}
\authornote{Corresponding author}
\affiliation{%
 \institution{Southwestern University of Finance and Economics}
  \city{Chengdu}
  \state{Sichuan}
  \country{China}}
\author{Guisong Liu}
\affiliation{%
 \institution{Southwestern University of Finance and Economics}
  \city{Chengdu}
  \state{Sichuan}
  \country{China}}
\author{Tao Gu}
\affiliation{%
 \institution{Southwestern University of Finance and Economics}
  \city{Chengdu}
  \state{Sichuan}
  \country{China}}
\author{Gang Kou}
\affiliation{%
  \institution{Hunan University of Technology and Business}
  \institution{Xiangjiang Laboratory}
  \state{Hunan}
  \country{China}}

\author{Carl Yang}
\authornotemark[1]
\affiliation{%
  \institution{Emory University}
  \city{Atlanta}
  \state{Georgia}
  \country{United States}}
\email{j.carlyang@emory.edu}

\renewcommand{\shortauthors}{Huaming Du et al.}

\begin{abstract}
Revealing the underlying causal mechanisms in the real world is crucial for scientific and technological progress. Despite notable advances in recent decades, the lack of high-quality data and the reliance of traditional causal discovery algorithms (TCDA) on the assumption of no latent confounders, as well as their tendency to overlook the precise semantics of latent variables, have long been major obstacles to the broader application of causal discovery. To address this issue, we propose a novel causal modeling framework, \textbf{TLVD}, which integrates the metadata-based reasoning capabilities of large language models (LLMs) with the data-driven modeling capabilities of TCDA for inferring latent variables and their semantics. Specifically, we first employ a data-driven approach to construct a causal graph that incorporates latent variables. 
Then, we employ multi-LLM collaboration for latent variable inference, modeling this process as a game with incomplete information and seeking its Bayesian Nash Equilibrium (BNE) to infer the possible specific latent variables. Finally, to validate the inferred latent variables across multiple real-world web-based data sources, we leverage LLMs for evidence exploration to ensure traceability. We comprehensively evaluate TLVD on three de-identified real patient datasets provided by a hospital and two benchmark datasets. Extensive experimental results confirm the effectiveness and reliability of TLVD, with average improvements of 32.67\% in Acc, 62.21\% in CAcc, and 26.72\% in ECit across the five datasets.

\end{abstract}




\maketitle
\section{Introduction}
Causal discovery aims to identify causal relationships from observational data and has been successfully applied in many fields \cite{pearl2019seven,du2025causal}. However, traditional methods—such as the PC algorithm \cite{spirtes2000causation}, GES \cite{chickering2002optimal}, and LiNGAM \cite{shimizu2006linear}—typically assume that no latent confounders exist in the causal graph, an assumption that often does not hold in many real-world scenarios. Therefore, extensive research has been devoted to addressing this issue from two directions to improve causal structure learning.

The first line of research focuses on inferring causal structures among observed variables despite the possible existence of latent confounders. Representative approaches include FCI and its variants based on conditional independence tests \cite{pearl2009causality,akbari2021recursive}, as well as over-complete ICA-based techniques that further exploit non-Gaussianity \cite{salehkaleybar2020learning}.
\begin{figure}[t] 
\includegraphics[width=0.48\textwidth]{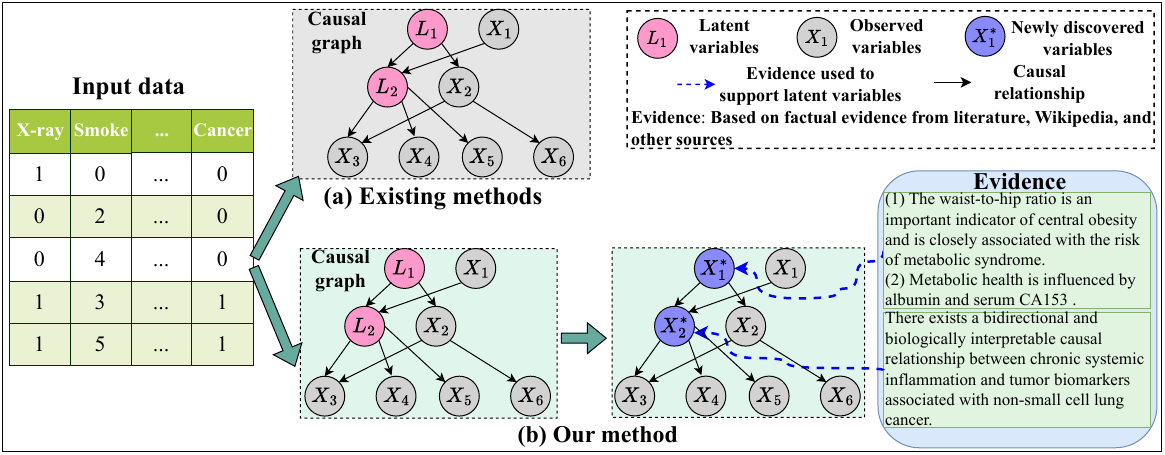}
  \centering
  \caption{A toy example of latent variable discovery using tabular data.}
  \Description{figure1}
  \label{fig:toy}
  \vspace{-1em}
\end{figure}
The second line of research focuses more on inferring causal structures among latent variables under the assumption that observed variables are not directly connected. This category includes Tetrad condition-based approaches \cite{kummerfeld2016causal}, high-order moments-based methods \cite{xie2020generalized,chen2022identification}, matrix decomposition-based approaches \cite{anandkumar2013learning}, and mixture oracle-based methods\cite{kivva2021learning}. Recently, Dong et al. \cite{dong2024versatile} proposed a three-phase causal discovery algorithm based on rank constraints that can identify the complete causal structure involving both observed and latent variables. However, as shown in Figure \ref{fig:toy}(a), existing studies focus on recovering causal structures involving latent variables but rarely infer the specific latent variables and their semantics.

In recent years, large language models (LLMs) \cite{achiam2023gpt,guo2025deepseek} have achieved remarkable breakthroughs in natural language understanding and generation, marking a critical milestone in the pursuit of artificial general intelligence. As the capabilities of LLMs continue to advance, LLM-based agents \cite{qin2024toolllm,wang2024survey} are increasingly becoming core components for integrating domain expertise and tools, effectively transforming these technological advances into practical applications. Building upon this paradigm, multi-agent systems \cite{li2025agent,ren2025towards} incorporate multiple diverse agents to coordinate and leverage their respective strengths, offering high flexibility and adaptability to provide more comprehensive solutions to complex real-world problems. The rapid development of LLMs has also provided new solutions for causal discovery \cite{zheng2024mulan}.

This paper focuses on a more challenging problem: \textbf{traceable latent variable discovery}, i.e., \textit{identifying specific latent variables and retrieving supporting evidence from multiple web data sources}. This setting is more general and more practical for handling many real-world problems, such as uncovering unknown factors contributing to diseases. This challenge involves two fundamental questions:
(i) How can multiple LLM-based agents be coordinated to efficiently propose and define possible latent variables?
(ii) How can evidence be retrieved from diverse web data sources (e.g., academic websites, Wikipedia, open databases, etc.) to consolidate the discovered latent variables?

To address this challenging problem, we propose a multi-agent collaboration framework for latent variable discovery, which simultaneously leverages real web data sources for evidence tracing in (causal) knowledge discovery. Specifically, we first employ latent variable causal discovery algorithms to recover causal graphs containing latent variables. Then, based on the recovered graph structure, we introduce game-theoretic principles, where multiple agents maintain their own belief networks to enable efficient collaboration, and utilize Bayesian Nash Equilibrium to hypothesize latent variables. Finally, we utilize LLMs to iteratively retrieve and exploit causal evidence across multiple real-world web data sources.

We conducted extensive experiments on our own real-world hospital dataset WCHSU as well as two benchmark datasets. Comprehensive evaluation results demonstrate that TLVD exhibits clear advantages over existing SOTA methods. The experimental results and discussions, ablation studies, parameter analyses, case studies, and failure analyses, will be presented and analyzed in Section \ref{sec:exp}.

The main contributions of this paper are as follows:

\noindent $\bullet$ A unified latent variable discovery framework called TLVD, which integrates LLMs with TCDA and employs multi-agent collaboration to infer more plausible and semantically interpretable latent variables.

\noindent $\bullet$ First, we incorporate the prior knowledge of LLMs, allowing them to serve as priors for latent variable explanation. Then, we model the multi-agent collaboration process as a game of incomplete information to efficiently reason about latent variables. Meanwhile, causal evidence is retrieved from multiple web data sources, and the latent variables are iteratively updated and validated.

\noindent $\bullet$  Extensive experimental results on various real-world medical and generic benchmark datasets demonstrate the superiority of TLVD.

\section{Related work}
\subsection{Causal Discovery}
Traditional and LLM-based causal discovery algorithms \cite{le2024multi,du2025causal} usually assume that all task-relevant variables are observable \cite{peters2014causal,mokhtarian2025recursive}. However, latent variables are prevalent in practice and can lead these methods to produce spurious causal relations, which has motivated extensive research on causal discovery with latent variables. Existing methods for handling latent variables in causal discovery can be broadly categorized into nine classes: Conditional independence constraints, Tetrad condition, Over-complete independent component analysis, Generalized independent noise, Mixture oracles-based, Rank deficiency, Heterogeneous data, and Score-based methods. Many methods fall within the constraint-based framework, combining conditional independence tests with algebraic constraints to infer causal relations, with representative approaches based on rank or tetrad constraints \cite{huang2022latent,dong2024versatile}. While most current methods still follow the constraint-based paradigm, some recent studies have started to formalize score-based approaches for latent causal discovery \cite{zhu2024causal,ng2024score}. Although these methods have advanced the field of causal discovery, none of them attempt to identify the specific latent variables and their semantics.

\subsection{LLM-based Multi-Agent Systems}
LLMs have exhibited remarkable proficiency in tackling various sophisticated tasks. However, they still suffer from several inherent limitations
, such as hallucination \cite{huang2025survey}, their autoregressive nature (e.g., inability to engage in slow thinking \cite{hagendorff2023human}), and constraints imposed by scaling laws. Recent studies have extensively explored multi-agent collaboration frameworks based on LLMs, aiming to tackle complex cognitive and decision-making tasks \cite{zhao2025sirius,yi2025from}. One prominent paradigm involves explicit role-playing mechanisms to simulate human collaborative dynamics, assigning different LLM agents to specialized roles within an organization \cite{hong2024metagpt}.
Other frameworks further enhance multi-agent collaboration through voting and consensus mechanisms \cite{zhangmore}, collective reasoning or discussion-based methods \cite{chen2024reconcile}, and structured (e.g., graph \cite{lu2024nodemixup,lu2024skipnode}-based) agentic debate approaches \cite{du2023improving,zhang2025g}, aiming to improve factual accuracy and logical consistency. However, current popular multi-LLM collaboration frameworks lack solid theoretical foundations, offer no guarantees of convergence or cooperation, and have yet to be explored in the field of causal discovery.

\section{Traceable Latent Variable Discovery Framework}
In this section, building upon game theory and reinforcement learning, we propose an evidence-traceable multi-agent collaboration framework for exploring latent variables and their semantics. As illustrated in Figure \ref{framework}, the framework primarily consists of three stages: (i) Stage I: Identification of latent causal graph structures; (ii) Stage II: Utilization of multi-agent collaboration to identify latent variables; (iii) Stage III: Validation of latent variables. It is important to note that this paper treats LLMs as large-scale background knowledge providers, while the core identifiability guarantees still rely on classical latent variable discovery algorithms.
\begin{figure*}[htbp] 
\includegraphics[width=0.97\textwidth]{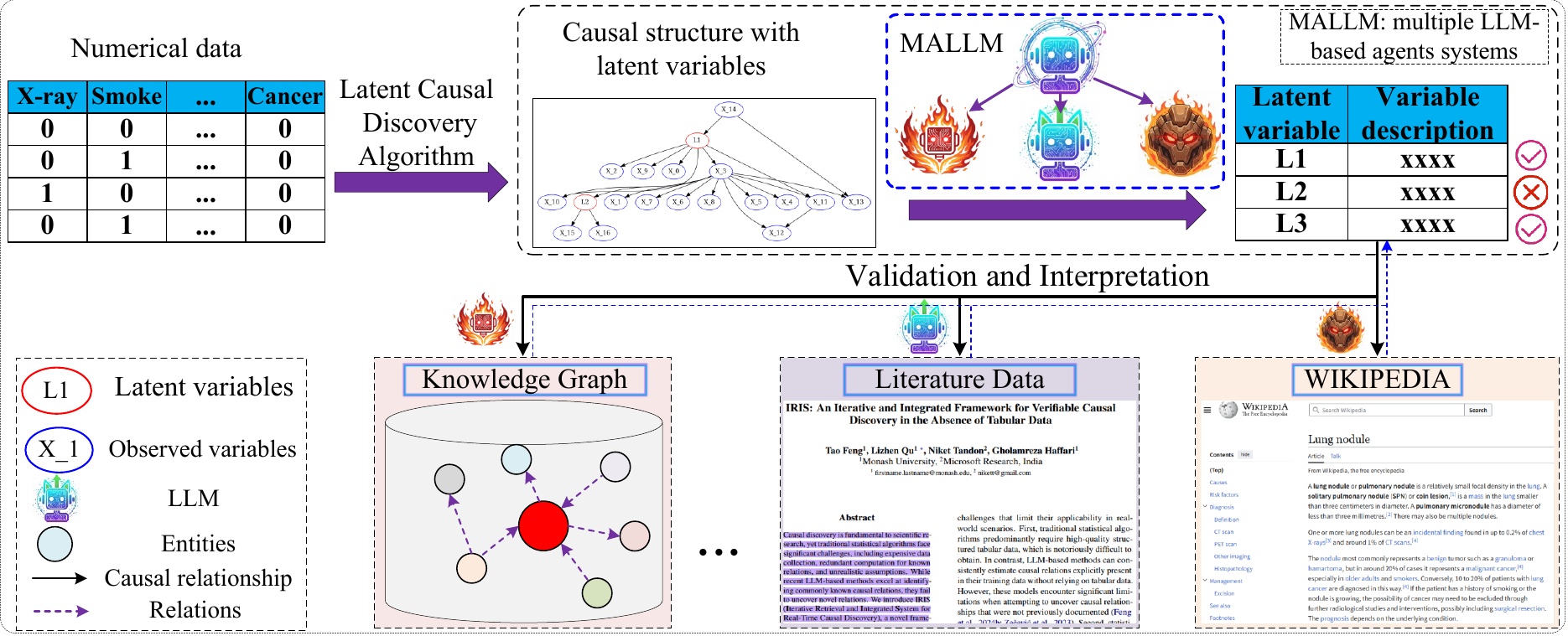}
  \centering
  \caption{The overview of TLVD framework.}
  \Description{framework.}
  \label{framework}
\end{figure*}
\subsection{Identifying Latent Causal Graph Structures}
In this paper, we aim to address a more general scenario for latent variable causal discovery, where observed variables can be directly adjacent, and latent variables can flexibly be related to all other variables. That is, hidden variables can serve as confounders, mediators, or effects of latent or observed variables, and even form a hierarchical structure (see Figure \ref{fig:toy}). This setup is quite general and practically meaningful for addressing many real-world problems. RLCD \cite{dong2024versatile} is currently the most advanced algorithm with theoretical guarantees, and thus, we use it to identify latent causal structures. In addition, we also explore the impact of other latent causal structure identification methods on the performance of our framework (see Appendix \ref{addition}). The specific formulation is as follows:
\begin{equation}
\mathcal{G}^{'}=RLCD\left(\mathbb{X}_{\mathcal{G}}\right)\ ,
\end{equation}
where $\mathbb{X}_{\mathcal{G}}$ denotes the data samples containing $n$ observed variables, and $\mathcal{G}^{'}$ represents the Markov equivalence class. 
\subsection{Identifying Latent Variables}\label{identify_lv}
Previous research \cite{huang2022latent,dong2024versatile}has primarily focused on revealing the positional and structural information of latent variables, while overlooking the discovery of the latent variables themselves and their semantics. However, identifying latent variables is crucial for causal inference and its practical applications. Although LLMs are capable of performing complex tasks such as creative writing, reasoning, and decision-making, with abilities that in some aspects even surpass human level, they are still constrained by hallucinations \cite{huang2025survey}, their autoregressive nature (e.g., inability to perform slow thinking \cite{hagendorff2023human}), and scaling laws \cite{ruan2024observational}. Inspired by the theory of mind \cite{frith2005theory}, multiple LLM-based agents systems have been developed, enabling teamwork and specialization by integrating the strengths and perspectives of individual agents to achieve shared goals. Therefore, in this paper, we adopt LLM-based multi-agent collaboration to identify latent variables and their semantics.
\begin{equation}
    \mathcal{X}_{i}=MALLM\left(x_{i},L_{i}\right)\ ,
\end{equation}
where $MALLM$ is the multi-agent collaboration system designed in this work, $L_{i}$ denotes the $i$-th latent variable, $x_{i}$ represents the Markov blanket corresponding to the $L_{i}$ in $\mathcal{G}^{'}$, and $\mathcal{X}_{i}$ denotes the concrete variable with explicit semantics corresponding to $L_{i}$.

However, achieving efficient collaboration among agents in MALLM is not straightforward and faces \emph{three key challenges}. First, extensive inter-agent communication consumes a large number of tokens, increasing computational overhead \cite{du2023improving}. Second, the volume of information exchanged over multiple rounds can exceed the context-window capacity of LLMs, limiting system scalability \cite{liu2024groupdebate}. Third, without well-defined coordination protocols, these systems may underperform compared to simple ensembling or self-consistency methods \cite{liang2024encouraging}. To address these challenges, we innovatively propose a novel approach for efficient coordination via Bayesian Nash Equilibrium (BNE), modeling multi-LLM interactions as an incomplete-information game to identify latent variables and their semantics.

Next, we first define the process, then provide a detailed introduction to the MALLM, and finally conduct the theoretical analysis.

\subsubsection{Process Definition}
We consider a system composed of $N-1$ execution LLMs and a coordinator LLM, where agents coordinate and interact solely based on their beliefs. Consistent with prior studies \cite{liang2025everyone}, we formally model this as a decentralized partially observable Markov decision process (DEC-POMDP), defined as a Markov game $\left \langle \mathcal{S}, \mathcal{A}, O, \mathcal{P}, \Omega, \mathcal{R}, \gamma \right \rangle $. Here, $\mathcal{S}$ represents the state space, including user queries and dialogue context; $\mathcal{A} = \mathcal{A}_1 \times \dots \times \mathcal{A}_N$ is the joint action space, with each $\mathcal{A}_i$ defining agent $i$’s action as a prompt embedding $a_i = [T_i, p_i]$, which controls the LLM’s output behavior via temperature and repetition penalty; $N$ is the number of agents; $O$ is the joint observation space; $\mathcal{P}$ and $\Omega$ define the state transition and observation functions; $\mathcal{R}$ is the reward, and $\gamma$ is the discount factor.

Our objective is to identify a policy profile $\pi = (\pi_1, \dots, \pi_N)$ that forms a BNE through belief coordination, such that no individual agent can improve the quality of its latent-variable generation and semantic inference by unilaterally changing its policy.

\subsubsection{BNE Implementation with MALLM}\label{MLLM}
In this section, we introduce a MALLM framework compliant with DEC-POMDP. The MLLM adopts a hierarchical architecture where multiple executor LLMs operate locally under the guidance of a coordinator LLM. The framework consists of two stages: Inference (see Figure \ref{agent-reason}) and Optimization (see Figure \ref{agent-framework}). Below, we provide a detailed introduction to the core modules.

\noindent\textbf{(1) Belief Update.} Each execution LLM $i$ maintains a belief network $B_i(\tau^t_i, o^t_i; \theta^B_i)$, which implements its policy $\pi_i$ by mapping the local trajectory $\tau^t_i$ (composed of its previous actions and observations) and current observation $o^t_i \in O_i$ into a belief state $b_i \in \mathbb{R}^d$. The belief state characterizes the agent’s understanding of the environment and the behaviors of other agents under partial observability, enabling strategic decision-making without direct access to others’ outputs. The belief state $b_i$ is further used to generate the prompt embedding $e_{i} = \left [ T_{i}
, p_{i} \right ]$, which controls the LLM’s output regarding the latent variables and their semantics. The specific formulation is as follows: 
\begin{equation}
\begin{aligned}
T_{i} &= \sigma \left ( W_{T} b_{i} + b_{T} \right ) \\
p_{i} &= \sigma \left ( W_{p} b_{i} + b_{p} \right )\ ,
\end{aligned}
\end{equation}
where $\sigma(\cdot)$ is the sigmoid function. $W_T$ and $b_T$ are learnable parameters. The belief network $B_i$ produces two outputs: (1) The prompt embedding $e_i$, which serves as the action in the DEC-POMDP framework; (2) The local Q-value $
Q^t_i(\tau^t_i, e^t_i; \varphi_i)
$, which estimates the expected return from the current belief state.
The belief state $b_i$ is also passed to a belief encoder for group-level processing.

To optimize the belief network, we apply a Temporal Difference (TD) loss to update the parameters $\theta^B_i = \left \{\varphi_i, W_T, b_T, W_p, b_p \right \} $, which include the Q-value function parameters $\varphi_i$ and the prompt embedding parameters, as illustrated in Figure \ref{agent-framework}. The TD loss is defined as follows:
\begin{equation}
\scalebox{0.84}{$
   \mathcal{L}_{TD}^i(\theta^B_i) = \mathbb{E}_\mathcal{D} \left[ \left( r^t_i + \gamma \;\underset{{e^{t+1}_i}}{\max} Q^{t+1}_i(\tau^{t+1}_i, e^{t+1}_i; \varphi'_i) - Q^t_i(\tau^t_i, e^t_i; \varphi_i) \right)^2 \right]$}\ , 
\end{equation}
where $r^t_i = \mathcal{R}(s^t, a^t)_i$ denotes the local reward signal, and $\varphi'_i$ represents the target network parameters updated via soft update mechanism. By minimizing $\mathcal{L}_{TD}^i$, execution LLM $i$ refines its belief state to improve local decision-making. To enable agents to implicitly exchange beliefs and thereby facilitate system convergence to BNE, we employ an attention mechanism within the belief encoder to capture inter-agent dependencies in the belief states:
\begin{equation}
    \mathcal{B} = \text{Attention} (W^Q \mathbf{b}, W^K \mathbf{b}, W^V \mathbf{b};\theta _{e})\ ,
\end{equation}
where $\mathbf{b} = [b_1; \dots; b_N] \in \mathbb{R}^{Nd}$ is the concatenated vector of the individual belief states $\{b_i\}_{i=1}^N$. The final output is: $E =  \mathcal{B} W^O$, where $\{W^Q, W^K, W^V\}$ are learnable parameters and $W^O$ is the output projection. The belief encoder captures high-level interactions among execution LLMs, ensuring coherence in group behavior.
\begin{figure}[t] 
\includegraphics[width=0.48\textwidth]{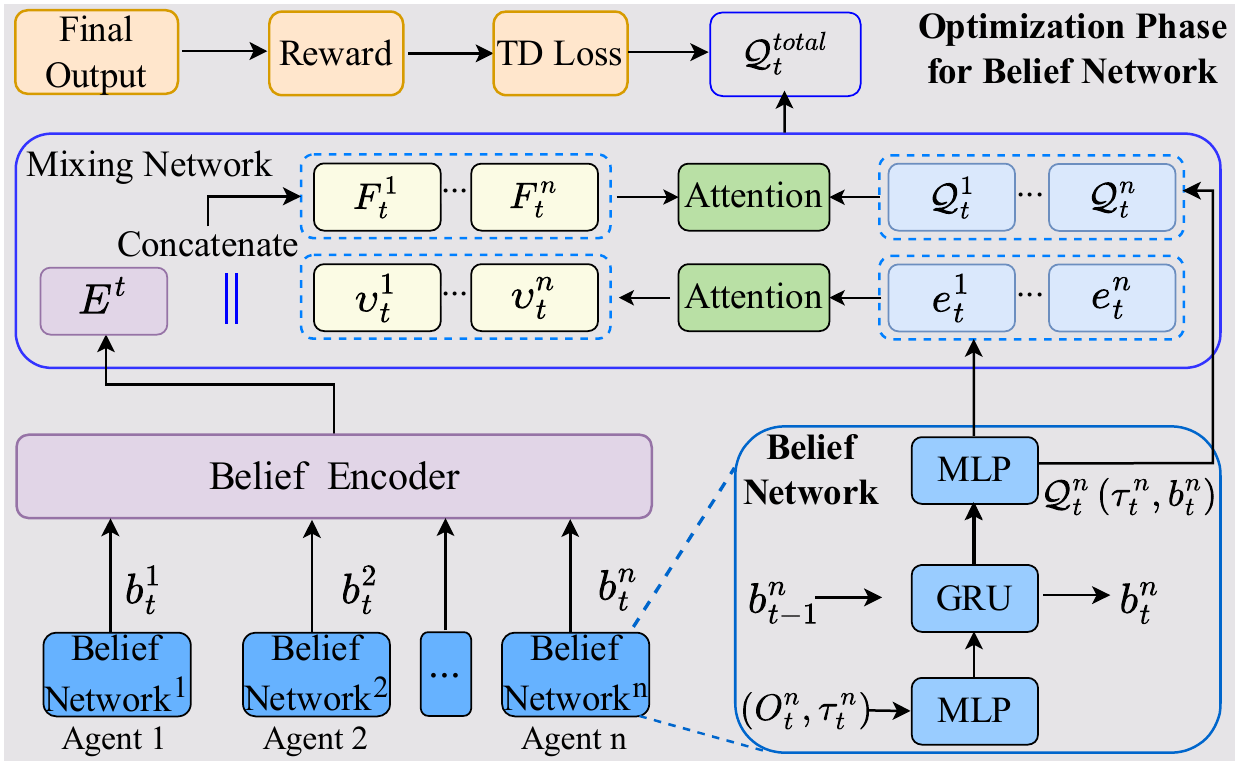}
  \centering
  \caption{The train process of MALLM.}
  \Description{Optimization Phase for Belief Network.}
  \label{agent-framework}
\end{figure}

\noindent\textbf{(2) Mixing Network.} Inspired by Qmax \cite{rashid2018qmix}, the mixing network coordinates the integrated belief information from all execution LLMs, thereby driving the overall system toward optimization with respect to the BNE. Specifically, each agent’s prompt embedding $\{e_t^i\}_{i=1}^N$ is first processed through a self-attention mechanism to capture inter-agent dependencies, producing intermediate embeddings $\{\upsilon_t^i\}_{i=1}^N$. These intermediate embeddings $\{\upsilon_t^i\}_{i=1}^N$ are then combined with the group-level belief representation $E_t$ to produce feature transformations $\{F_t^i\}_{i=1}^N$. This combination enables the network to jointly integrate individual belief information (captured in $e_t^i$) and collective belief dynamics (captured in $E_t$), thereby facilitating coordinated optimization. Next, the local Q-values $\{Q_t^i\}_{i=1}^N$ and the transformed features $\{F_t^i\}_{i=1}^N$ are fed together into multi-head attention layers to compute the global Q-value $Q_t^{tot}$.
This global Q-value function accounts for local–global interactions, ensuring that improvements in individual behaviors also contribute to overall performance enhancement. To train the mixing network, the following loss function is minimized:
\begin{equation}
\begin{aligned}
  \mathcal{L}_{\text{mix}}\left(\phi\right)=& \mathbb{E}_{\mathcal{D}}\Big[\big(r_{tot}+\gamma\max_{\{e^{t+1}_i\}_{i=1}^N}Q^{t+1}_{tot}(\tau_{t+1},\{e^{t+1}_i\}_{i=1}^N;\phi')
 - \\&Q^t_{tot}(\tau_t,\{e^t_i\}_{i=1}^N;\phi)\big)^2\Big]+\lambda_m \sum_{i=1}^N\|Q_i^t-Q^t_{tot}\|^2\ ,  
\end{aligned}
\end{equation}
the term $\|Q^t_i - Q^t_{tot}\|^2$ ensures that local Q-values remain consistent with the global estimate. In addition, the target network parameters $\phi'$ are updated using a soft update rule: $\phi' \leftarrow \tau \phi + (1 - \tau)\phi'$. Through this mechanism, the mixing network can optimize local policies to improve the global objective, thereby promoting stable convergence during training.

\noindent\textbf{(3) Reward Design.}
The reward function $\mathcal{R}_{\text{design}}$ consists of the following \underline{four components}. \textit{(a) Action Likelihood Reward} $r^{AL}_i = \min(R_{max}, \text{sim}(u_i, C))$
measures the consistency of the final output with the target via cosine similarity $\text{sim}(u_i, C) = \frac{u_i \cdot C}{\|u_i\| \|C\|}$, where $C$ denotes the output of the coordinator and $u_i$
 represents the output of an individual executor. \textit{(b) Uncertainty Reduction Reward}
$r^{UR}_i = \mathcal{C}on_{i}\cdot \text{sim}(u_i, C)$ 
assesses the uncertainty of LLMs by considering both the model’s direct confidence and the similarity to the overall answer, where $\mathcal{C}on_{i}$ represents the confidence of agent $i$. \textit{(c) Collaborative Contribution Reward} $r^{CC}_i = \min(R_{max}, \text{quality}(u_i, \{u_j\}_{j \neq i}))$ assesses each agent’s contribution to the collective solution \cite{xie2023self}. \textit{(d) Evidence Reliability Reward} $r^{ER}_i=CAcc_{i}$ measures the reliability of the evidence corresponding to each query result across multiple web data sources, where $CAcc$ represents the evidence ratio. For further details, see Appendix \ref{data}.
The total reward is computed as: 
$r_i = \alpha_1 r^{AL}_i + \alpha_2 r^{UR}_i + \alpha_3 r^{CC}_i+ \alpha_4 r^{ER}_i$, where $\alpha_1 + \alpha_2 + \alpha_3+ \alpha_4 = 1$. 
\begin{figure}[t] 
\includegraphics[width=0.48\textwidth]{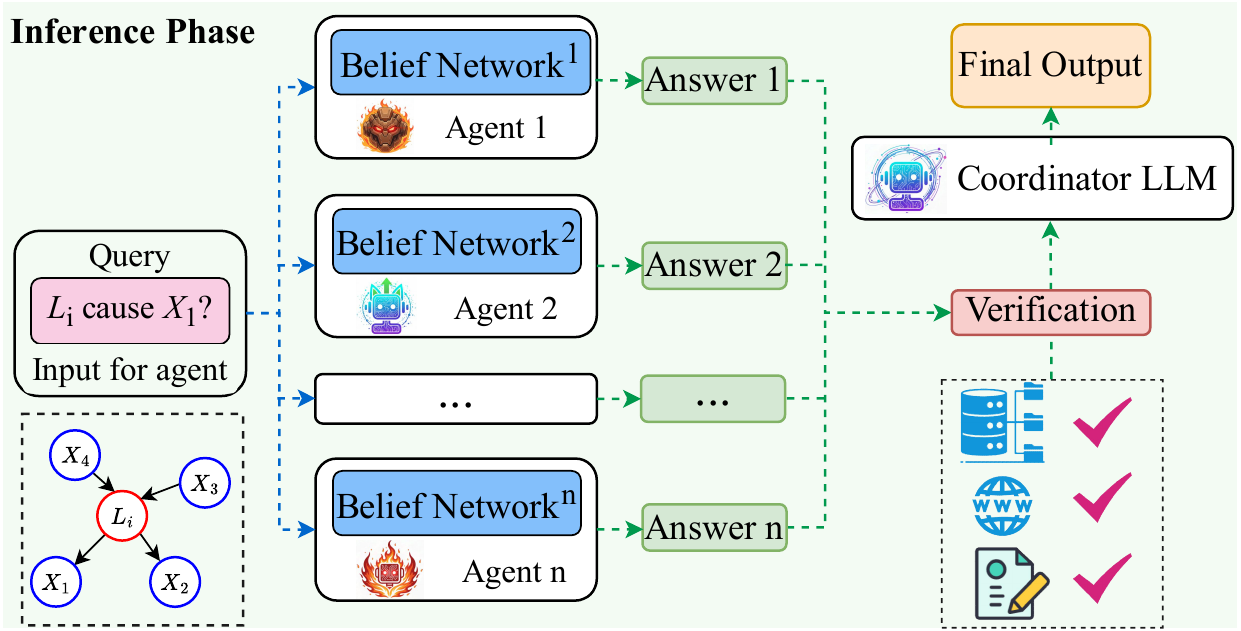}
  \centering
  \caption{The reasoning process of MALLM.}
  \Description{Optimization Phase for Belief Network.}
  \label{agent-reason}
\end{figure}

\noindent\textbf{(4) Inference Phase.}
During the inference phase, the coordinator LLM distributes query information containing latent variables to multiple execution LLMs, which independently generate their own answers. By incorporating the retrieved evidence from Section \ref{sec:verfi}, the coordinator LLM aggregates both the answers and the evidence to produce the final output, as shown in Figure \ref{agent-reason}. 
\subsubsection{Theoretical Analysis}
In this subsection, we mainly analyze the existence and convergence of BNE.

\noindent\textbf{(1) Existence of BNE.}
To bridge the DEC-POMDP formulation with BNE analysis, we treat each agent $i$ as a player, where each player forms and updates beliefs about other agents’ types based on a common prior and its own observations. Each player’s type $\theta_i$ is determined by their internal beliefs and private observation history. Formally, if there exists a strategy profile $\{\pi^*_i\}_{i=1}^N$ such that for any player $i$:
\begin{equation}
\scalebox{0.90}{$
    \mathbb{E}_{\theta_{-i}}
\big[
U_i(\pi^*_i(\theta_i),\ \pi^*_{-i}(\theta_{-i}))
\big]
\ge
\mathbb{E}_{\theta_{-i}}
\big[
U_i(\pi'_i(\theta_i),\ \pi^*_{-i}(\theta_{-i}))
\big],
\quad \forall \pi'_i\ ,$}
\end{equation}
where $\theta_{-i} = (\theta_1,\dots,\theta_{i-1},\theta_{i+1},\dots,\theta_N)$ denotes the types of all agents except $i$.
The utility function is defined as:
\begin{equation}
    U_i(\pi^*_i, \pi^*_{-i}, \theta_i, \theta_{-i}) =
\mathbb{E}\left[\sum_{t=0}^{\infty} \gamma^{t} r^t_i| \pi^*_i, \pi^*_{-i}, \theta_i, \theta_{-i}\right]\ ,
\end{equation}
where $r^t_i = R(s^t, a^t)_i$ denotes the reward obtained by player $i$ at time $t$, $N$ denotes the number of players. 

\begin{theorem}\label{the1}
(Existence of BNE) In our MALLM framework, assuming certain conditions \cite{yi2025from} hold, then by Glicksberg’s Fixed Point Theorem \cite{ahmad2023common}, there exists a BNE strategy profile
$\pi^\ast = (\pi^*_1,\dots,\pi^*_N)$. A complete proof is provided in Appendix \ref{app-them1}.
\end{theorem}
\noindent\textbf{(2) Convergence of BNE.}
We analyze the convergence of the TLVD framework through Bayesian regret. First, to connect the theoretical analysis with the DEC-POMDP formulation, we define the Bayesian regret of each agent $i$ over $T$ steps as:
\begin{equation}
    R_i(T)=
\mathbb{E}_{s_t,\pi_t}
\left[
\sum_{t=1}^T \big(V^*_i(s_t)-V^{\pi_t}_i(s_t)\big)
\right]\ ,
\end{equation}where the optimal value function under BNE is:
\begin{equation}
    V^*_i(s)=
\max_{\pi_i}
\mathbb{E}_{\pi^*_{-i}}
\left[
\sum_{t=0}^\infty \gamma^t \mathcal{R}(s_t,a_t)_i
\mid s_0=s,\ \pi_i,\ \pi^*_{-i}
\right]\ ,
\end{equation}
and $V^{\pi_t}_i(s)$ is the value function under the current strategy profile $\pi^t = (\pi^t_1,\dots,\pi^t_N)$ at time $t$.
The expectation accounts for randomness in both state transitions (governed by $P$) and policy selections. The total Bayesian regret is therefore defined as:
$R(T) = \sum_{i=1}^N R_i(T)$.

Second, we introduce standard assumptions \cite{yi2025from} and propose Lemma \ref{them2}, which provides an upper bound on the Bayesian regret. A proof sketch is given below, with more details provided in Appendices \ref{convergence}, \ref{proof-them2}, and \ref{regret}.

\begin{lemma}\label{them2}
(Performance Difference)
For joint policies $\pi = (\pi_i, \pi_{-i})$ and $\pi' = (\pi'_i, \pi'_{-i})$, the difference in their value functions satisfies:
\begin{equation}
\scalebox{0.92}{$
V^{\pi'}_i(s) - V^{\pi}_i(s)=
\frac{1}{1 - \gamma}
\ \mathbb{E}_{s\sim d_{\pi'}}\Big[
\mathbb{E}_{a\sim \pi'} Q^{\pi}_i(s, a)
-
\mathbb{E}_{a\sim \pi} Q^{\pi}_i(s, a)
\Big]\ ,$}
\end{equation}
where $d_{\pi'}$ is the state distribution under policy $\pi'$, and $a=(a_i,a_{-i})$ denotes the joint action from the action space $\mathcal{A}$.
\end{lemma} 
Note that the function $Q^{\pi}_i(s,a)$ in this lemma will be approximated by neural networks in our implementation (see Section \ref{MLLM} for details). Consistent with existing studies \cite{fujimoto2018addressing,yi2025from}, we apply this lemma to our regret analysis and obtain:
\begin{equation}
\scalebox{0.82}{$
        R(T)=
\sum_{i=1}^{N}
\frac{1}{1-\gamma}
\ \mathbb{E}_{s_t,\pi_t}\!\Bigg[
\sum_{t=1}^T
\Big(
\mathbb{E}_{a^*_t\sim \pi^*} Q^{\pi_t}_i(s_t,a^*_t)
-\mathbb{E}_{a_t\sim \pi_t} Q^{\pi_t}_i(s_t,a_t)
\Big)
\Bigg]$}\ ,
\end{equation}
where $\pi^*$ denotes the BNE policies. By bounding the suboptimality $\delta_t$, we obtain:
\begin{equation}
    \mathbb{E}_{a^*_t\sim \pi^*} Q^{\pi_t}_i(s_t,a^*_t)-\mathbb{E}_{a_t\sim \pi^t} Q^{\pi_t}_i(s_t,a_t)
\;\le\; \delta_t\ ,
\end{equation}
where $\delta_t = O(1/\sqrt{t})$ bounds policy suboptimality, $O(\cdot)$ is the asymptotic upper bound. Under standard regularity conditions, these errors can be bounded by constants $C_\delta$, respectively, which yields:
\begin{equation}
    R(T)
\le
\sum_{i=1}^N \frac{1}{1-\gamma}
C_\delta
\sum_{t=1}^T \frac{1}{\sqrt{t}}
=
O\!\Big(\frac{N\sqrt{T}}{1-\gamma}\Big).
\end{equation}



\subsection{Verification of Latent Variables}\label{sec:verfi}
Despite the ability of multiple agents in TLVD to infer potential latent variables, further validation is necessary to ensure their plausibility and to achieve the traceability of the proposed method. Therefore, we leverage data sources (e.g., arXiv, Wikipedia, databases, and patients’ anonymized personal text reports) for verification, and feed the retrieved evidence back to the execution LLMs and the coordinator LLM to update their strategies and beliefs. 
\subsection{Complexity Analysis}
The space complexity of our proposed TLVD mainly depends on the belief networks, belief encoder, and mixing network. For the belief networks, the complexity is $\mathcal{O}(Nd{d}')$, where $d$ denotes the input dimension, and ${d}'$ denotes the hidden layer dimension; for the belief encoder, the complexity is $\mathcal{O}(N{{d}'}^2)$; and for the mixing network, the complexity is $\mathcal{O}({{d}'}^2)$. Therefore, the space complexity of TLVD can be summarized as $\mathcal{O}(N{{d}'}^2)$.

\section{Experiments}\label{sec:exp}
Our study primarily focuses on the following research questions: \textbf{RQ1}: How does the performance of TLVD compare with existing methods?
\textbf{RQ2}: What is the impact of model configurations on overall performance? \textbf{RQ3}: How does each component of TLVD affect the overall performance? \textbf{RQ4}: How do hyperparameter settings affect the performance of TLVD? \textbf{RQ5}: How does TLVD operate in real-world examples? \textbf{RQ6}: What is the impact of web data sources on TLVD, and what are the main causes of its failures?
\subsection{Experimental Setup}
\subsubsection{Datasets \& evaluations}
In this section, we evaluate our method on five
real-world datasets.  As shown in Table \ref{data-stcs}, we utilized the de-identified WCHSU-Cancer and WCHSU-Pain datasets from real hospitals scenarios. Additionally, we also use two generic domain benchmark datasets (Multitasking Behaviour Study \cite{himi2019multitasking} and Teacher’s Burnout Study \cite{byrne2013structural}). \textit{Please note that the WCHSU dataset is newly collected and has not been made publicly available online, making it unlikely to have been used in training any existing LLMs, and therefore posing no risk of data leakage}. The evaluation metrics are ACC, CAcc, and ECit. Detailed descriptions can be found in Appendix \ref{data}. For each dataset, we conduct five experiments with different random seeds and report the average performance.
\subsubsection{Baselines}
We compare four types of baselines: single LLMs, deep research agents, multi-agent platforms, and multi-LLM reasoning frameworks. The single LLMs mainly include GPT-5. The deep research agents include WideSearch \cite{wong2025widesearch}, Gemini-deepresearch, Openai-deepresearch, Qwen-deepresearch, and Doubao-deepresearch. The multi-agent platforms include Autogen \cite{wu2024autogen} and MiniMax, whereas the multi-LLM reasoning frameworks include CAMEL \cite{li2023camel}, Multiagent (Majority) \cite{zhangmore}, and Multiagent (Debate) \cite{du2023improving}.
\subsubsection{Implementation details}
In this section, the TLVD framework consists of one coordinator and two executor LLMs. During training, the episodes per task is set to 100, buffer size $ \left | \mathcal{D} \right |$ to 32, optimizer to Adam, learning rate to 0.001, discount factor to 0.99, entity dimension to 256, belief state dimension to 128, MLP hidden size to 256, and $\lambda_m$ to 0.1. $R_{max}$ is 1. For the early stopping mechanism, we utilize the same termination threshold settings as in previous studies \cite{yi2025from}. To ensure a fair comparison with baseline methods, we employed three identical models across these LLMs. For heterogeneous results, we also evaluated TLVD using different models, as shown in Table \ref{table_different_llms}. All evaluations were conducted under a zero-shot setting on a server equipped with an NVIDIA GeForce A6000 GPU with 48GB of memory. For the LLMs used to validate latent variables in web data sources, we employed the same LLMs as those used for the executors. For each causal validation query, we retrieved the top five most relevant texts from arXiv and Wikipedia each time. 
\subsection{RQ1: Comparisons and Analysis}
Table \ref{table-main} shows a performance comparison of different methods on the three WCHSU datasets. The results indicate that TLVD outperforms all baseline methods across all evaluation metrics. On average, in terms of ACC, TLVD achieves a 131.68\% improvement over MiniMax, a 124.30\% improvement over Autogen, and a 58.07\% improvement over OpenAI-deepresearch.
Moreover, when we constructed new datasets by randomly sampling the original data (results in Table \ref{table-main1}), TLVD still maintained the top performance: It outperformed the runner-up, OpenAI-deepresearch, with average improvement of 32.96\% in ACC, 239.72\% in CAcc, and 52.86\% in ECit. Notably, TLVD achieves performance improvements with fewer communication tokens compared to CAMEL, Multi-Agent Debate, and Multi-Agent Majority, which can be attributed to the design of our belief network (see Figure \ref{cost-parameter} in Appendix \ref{addition}).
\subsection{RQ2: Model Configuration Analysis}
To evaluate the impact of the coordinator LLM and execution LLM performance on the TLVD, and to investigate whether heterogeneous execution LLMs can also achieve a BNE, we conducted \underline{two types of experiments}: one pairing a strong coordinator LLM with weaker execution LLMs, and another pairing a weaker coordinator LLM with stronger execution LLMs. These experiments were further divided into homogeneous and heterogeneous execution groups for detailed analysis. To ensure a fair comparison, the coordinator LLM was consistently set to Kimi k2 32B across all experiments. For the heterogeneous execution group, we used the following configurations: LLaMA3.1 8B and Qwen2.5 7B; as well as another configuration consisting of Qwen 2.5 72B and LLaMA3.1 70B. For the homogeneous execution group, two configurations were tested: one with two weak models (LLaMA 3.1 8B) and another with two strong models (LLaMA 3.1 70B). As shown in Table \ref{table_different_llms}, stronger execution LLMs achieve BNE more efficiently by providing higher-quality answers. Additionally, heterogeneous models perform worse than homogeneous models due to the increased difficulty in reaching BNE.
\begin{table}[t]
\caption{Statistics of datasets.}
\centering
\resizebox{0.47\textwidth}{!}{
\begin{tabular}[t]{lccccc}
\toprule
\textbf{Dataset} &\textbf{\#Domain}&\textbf{\#Sample} & \textbf{\makecell[c]{\#Number of \\observed variables}}& \textbf{\makecell[c]{\#Number of \\latent variables}}\\
\midrule
\textbf{\textit{WCHSU-Cancer}} &Medical &200,000& 12/22 & 6/4  \\
\textbf{\textit{WCHSU-Pain}} &Medical &1,568&  16 & 4 \\
\textbf{\textit{\makecell[l]{Multitasking \\Behaviour Study}}} &Social Science&202& 9 & 4 \\
\textbf{\textit{\makecell[l]{Teacher’s \\Burnout Study}}} &Social Science&599& 32 & 11 \\
\bottomrule
\end{tabular}
}
\label{data-stcs}
\end{table}
\begin{table*}[t]
    \caption{Performance evaluation of different models.}
    \label{table-main}
    \centering
    \resizebox{0.99\textwidth}{!}{
    \newcommand{\tabincell}[2]{\begin{tabular}{@{}#1@{}}#2\end{tabular}}
    \centering
    \begin{tabular}{l||ccc|ccc|ccc}
     \toprule    
      &\multicolumn{3}{c|}{\bf WCHSU-Cancer ($n=12$)} &\multicolumn{3}{c|}{\bf WCHSU-Cancer ($n=22$)}&\multicolumn{3}{c}{\bf WCHSU-Pain}\\
      \textbf{Methods} &\textbf{ACC}&\textbf{CAcc}&\textbf{ECit}&\textbf{ACC}&\textbf{CAcc}&\textbf{ECit}&\textbf{ACC}&\textbf{CAcc}&\textbf{ECit}\\
     \midrule
     \midrule
        \textbf{GPT5} &0.134$\pm$\textit{\scriptsize 0.067} &0.100$\pm$\textit{\scriptsize 0.094} &0.031$\pm$\textit{\scriptsize 0.029}&0.150$\pm$\textit{\scriptsize 0.122}&0.116$\pm$\textit{\scriptsize 0.039}&0.079$\pm$\textit{\scriptsize 0.027} &0.150$\pm$\textit{\scriptsize 0.122}&0.095$\pm$\textit{\scriptsize 0.061}&0.095$\pm$\textit{\scriptsize 0.061}\\
    \midrule
        \textbf{Gemini-deepresearch}  &0.267$\pm$\textit{\scriptsize 0.081} &0.286$\pm$\textit{\scriptsize 0.046} &0.154$\pm$\textit{\scriptsize 0.024} &0.300$\pm$\textit{\scriptsize 0.100} &0.347$\pm$\textit{\scriptsize 0.042}&0.236$\pm$\textit{\scriptsize 0.029}&0.250$\pm$\textit{\scriptsize 0.000}&0.338$\pm$\textit{\scriptsize 0.031} &0.284$\pm$\textit{\scriptsize 0.026}\\
        \textbf{OpenAI-deepresearch} & 0.433$\pm$\textit{\scriptsize 0.082}&0.560$\pm$\textit{\scriptsize 0.049}&0.215$\pm$\textit{\scriptsize 0.019}&0.550$\pm$\textit{\scriptsize 0.100}&0.214$\pm$\textit{\scriptsize 0.165} &0.214$\pm$\textit{\scriptsize 0.089}&0.550$\pm$\textit{\scriptsize 0.100}&0.253$\pm$\textit{\scriptsize 0.050}&0.200$\pm$\textit{\scriptsize 0.039}\\
        \textbf{Qwen-deepresearch} &0.200$\pm$\textit{\scriptsize 0.066} &0.500$\pm$\textit{\scriptsize 0.000}&0.077$\pm$\textit{\scriptsize 0.000}&0.250$\pm$\textit{\scriptsize 0.000}&0.271$\pm$\textit{\scriptsize 0.029} &0.271$\pm$\textit{\scriptsize 0.029}&0.300$\pm$\textit{\scriptsize 0.100}&0.500$\pm$\textit{\scriptsize 0.158} &0.105$\pm$\textit{\scriptsize 0.033}\\
        \textbf{Doubao-deepresearch} &0.000$\pm$\textit{\scriptsize 0.000} &0.150$\pm$\textit{\scriptsize 0.050}&0.038$\pm$\textit{\scriptsize 0.070}&0.000$\pm$\textit{\scriptsize 0.000}&0.000$\pm$\textit{\scriptsize 0.000} &0.000$\pm$\textit{\scriptsize 0.000}&0.150$\pm$\textit{\scriptsize 0.122}&0.000$\pm$\textit{\scriptsize 0.000} &0.000$\pm$\textit{\scriptsize 0.000}\\
        \midrule
        \textbf{Autogen} &0.300$\pm$\textit{\scriptsize 0.125}&0.000$\pm$\textit{\scriptsize 0.000}&0.000$\pm$\textit{\scriptsize 0.000} &0.450$\pm$\textit{\scriptsize 0.100}&0.255$\pm$\textit{\scriptsize 0.036}&0.200$\pm$\textit{\scriptsize 0.029} &0.350$\pm$\textit{\scriptsize 0.122}&0.089$\pm$\textit{\scriptsize 0.044}&0.084$\pm$\textit{\scriptsize 0.042}\\
\textbf{MiniMax}&0.467$\pm$\textit{\scriptsize 0.067}&0.600$\pm$\textit{\scriptsize 0.000}& 0.231$\pm$\textit{\scriptsize 0.000}&0.300$\pm$\textit{\scriptsize 0.100}&0.192$\pm$\textit{\scriptsize 0.030}&0.171$\pm$\textit{\scriptsize 0.027}&0.300$\pm$\textit{\scriptsize 0.100}&0.071$\pm$\textit{\scriptsize 0.014}&0.053$\pm$\textit{\scriptsize 0.008}\\
    \midrule
        \textbf{\textsc{CAMEL}} & 0.267$\pm$\textit{\scriptsize 0.081} &0.000$\pm$\textit{\scriptsize 0.000} & 0.000$\pm$\textit{\scriptsize 0.000}&0.400$\pm$\textit{\scriptsize 0.122}&0.000$\pm$\textit{\scriptsize 0.000} &0.000$\pm$\textit{\scriptsize 0.000}&0.000$\pm$\textit{\scriptsize 0.000} &0.000$\pm$\textit{\scriptsize 0.000}&0.000$\pm$\textit{\scriptsize 0.000}\\
        \textbf{\textsc{Multi-Agent-Majority}} & 0.134$\pm$\textit{\scriptsize 0.067} &0.000$\pm$\textit{\scriptsize 0.000} & 0.000$\pm$\textit{\scriptsize 0.000}&0.200$\pm$\textit{\scriptsize 0.100}&0.000$\pm$\textit{\scriptsize 0.000} &0.000$\pm$\textit{\scriptsize 0.000}&0.000$\pm$\textit{\scriptsize 0.000} &0.000$\pm$\textit{\scriptsize 0.000}&0.000$\pm$\textit{\scriptsize 0.000}\\
        \textbf{\textsc{Multi-Agent Debate}} & 0.234$\pm$\textit{\scriptsize 0.081} &0.000$\pm$\textit{\scriptsize 0.000} & 0.000$\pm$\textit{\scriptsize 0.000}&0.350$\pm$\textit{\scriptsize 0.122}&0.000$\pm$\textit{\scriptsize 0.000} &0.000$\pm$\textit{\scriptsize 0.000}&0.200$\pm$\textit{\scriptsize 0.100} &0.000$\pm$\textit{\scriptsize 0.000}&0.000$\pm$\textit{\scriptsize 0.000}\\
        \midrule
        \textbf{\textsc{TLVD}} &\textbf{0.833$\pm$\textit{\scriptsize 0.000}}&\textbf{0.900$\pm$\textit{\scriptsize 0.033}} &\textbf{0.415$\pm$\textit{\scriptsize 0.015}}&\textbf{0.750$\pm$\textit{\scriptsize 0.000}}&\textbf{0.920$\pm$\textit{\scriptsize 0.040}}&\textbf{0.329$\pm$\textit{\scriptsize 0.014}}&\textbf{0.800$\pm$\textit{\scriptsize 0.100}}& \textbf{0.860$\pm$\textit{\scriptsize 0.049}}&\textbf{0.453$\pm$\textit{\scriptsize 0.026}}\\
    \bottomrule
    \end{tabular}}
\end{table*}
\begin{table*}[t]
    \caption{Performance evaluation of different models, where we randomly sample half of the original dataset.}
    \label{table-main1}
    \centering
     \resizebox{0.99\textwidth}{!}{
    \newcommand{\tabincell}[2]{\begin{tabular}{@{}#1@{}}#2\end{tabular}}
    \centering
    \begin{tabular}{l||ccc|ccc|ccc}
     \toprule    
      &\multicolumn{3}{c|}{\bf WCHSU ($n=12$)} &\multicolumn{3}{c|}{\bf WCHSU ($n=22$)}&\multicolumn{3}{c}{\bf WCHSU-Pain}\\
      \textbf{Methods} &\textbf{ACC}&\textbf{CAcc}&\textbf{ECit}&\textbf{ACC}&\textbf{CAcc}&\textbf{ECit}&\textbf{ACC}&\textbf{CAcc}&\textbf{ECit}\\
     \midrule
     \midrule
        \textbf{GPT5} &0.120$\pm$\textit{\scriptsize 0.098} &0.074$\pm$\textit{\scriptsize 0.042} &0.074$\pm$\textit{\scriptsize 0.042}&0.000$\pm$\textit{\scriptsize 0.000}&0.111$\pm$\textit{\scriptsize 0.070}&0.033$\pm$\textit{\scriptsize 0.021} &0.133$\pm$\textit{\scriptsize 0.163}&0.51$\pm$\textit{\scriptsize 0.0037}&0.36$\pm$\textit{\scriptsize 0.0056}\\
    \midrule
        \textbf{Gemini-deepresearch}  &0.240$\pm$\textit{\scriptsize 0.013} &0.084$\pm$\textit{\scriptsize 0.042} &0.084$\pm$\textit{\scriptsize 0.042} &0.200$\pm$\textit{\scriptsize 0.163} &0.160$\pm$\textit{\scriptsize 0.013}&0.160$\pm$\textit{\scriptsize 0.013}&0.266$\pm$\textit{\scriptsize 0.133}&0.000$\pm$\textit{\scriptsize 0.000} &0.000$\pm$\textit{\scriptsize 0.000}\\
        \textbf{OpenAI-deepresearch} & 0.600$\pm$\textit{\scriptsize 0.126}&0.388$\pm$\textit{\scriptsize 0.029}&0.347$\pm$\textit{\scriptsize 0.026}&0.667$\pm$\textit{\scriptsize 0.000}&0.180$\pm$\textit{\scriptsize 0.016} &0.180$\pm$\textit{\scriptsize 0.016}&0.600$\pm$\textit{\scriptsize 0.134}&0.318$\pm$\textit{\scriptsize 0.047}&0.318$\pm$\textit{\scriptsize 0.047}\\
        \textbf{Qwen-deepresearch} &0.440$\pm$\textit{\scriptsize 0.080}&0.225$\pm$\textit{\scriptsize 0.094}&0.095$\pm$\textit{\scriptsize 0.039}&0.200$\pm$\textit{\scriptsize 0.163}&0.289$\pm$\textit{\scriptsize 0.089}&0.087$\pm$\textit{\scriptsize 0.027}&0.266$\pm$\textit{\scriptsize 0.133}&0.000$\pm$\textit{\scriptsize 0.000}&0.000$\pm$\textit{\scriptsize 0.000}\\
        \textbf{Doubao-deepresearch} &0.000$\pm$\textit{\scriptsize 0.000} &0.000$\pm$\textit{\scriptsize 0.000}&0.000$\pm$\textit{\scriptsize 0.000}&0.133$\pm$\textit{\scriptsize 0.163}&0.000$\pm$\textit{\scriptsize 0.000} &0.000$\pm$\textit{\scriptsize 0.000}&0.067$\pm$\textit{\scriptsize 0.133}&0.000$\pm$\textit{\scriptsize 0.000} &0.000$\pm$\textit{\scriptsize 0.000}\\
        \midrule
        \textbf{Autogen} &0.480$\pm$\textit{\scriptsize 0.098}&0.000$\pm$\textit{\scriptsize 0.000}&0.000$\pm$\textit{\scriptsize 0.000} &0.333$\pm$\textit{\scriptsize 0.000} &0.186$\pm$\textit{\scriptsize 0.042} &0.173$\pm$\textit{\scriptsize 0.039}&0.400$\pm$\textit{\scriptsize 0.250}&0.000$\pm$\textit{\scriptsize 0.000}&0.000$\pm$\textit{\scriptsize 0.000}\\
\textbf{MiniMax}&0.560$\pm$\textit{\scriptsize 0.080}&0.000$\pm$\textit{\scriptsize 0.000}& 0.000$\pm$\textit{\scriptsize 0.000}&0.533$\pm$\textit{\scriptsize 0.164}&0.400$\pm$\textit{\scriptsize 0.122}&0.053$\pm$\textit{\scriptsize 0.016}&0.467$\pm$\textit{\scriptsize 0.267}&0.350$\pm$\textit{\scriptsize 0.062}&0.247$\pm$\textit{\scriptsize 0.044}\\
    \midrule
        \textbf{\textsc{CAMEL}} & 0.360$\pm$\textit{\scriptsize 0.080} &0.000$\pm$\textit{\scriptsize 0.000} & 0.000$\pm$\textit{\scriptsize 0.000}&0.000$\pm$\textit{\scriptsize 0.000}&0.000$\pm$\textit{\scriptsize 0.000} &0.000$\pm$\textit{\scriptsize 0.000}&0.266$\pm$\textit{\scriptsize 0.133} &0.000$\pm$\textit{\scriptsize 0.000}&0.000$\pm$\textit{\scriptsize 0.000}\\
        \textbf{\textsc{Multi-Agent-Majority}} & 0.160$\pm$\textit{\scriptsize 0.080} &0.000$\pm$\textit{\scriptsize 0.000} & 0.000$\pm$\textit{\scriptsize 0.000}&0.133$\pm$\textit{\scriptsize 0.163}&0.000$\pm$\textit{\scriptsize 0.000} &0.000$\pm$\textit{\scriptsize 0.000}&0.067$\pm$\textit{\scriptsize 0.133} &0.000$\pm$\textit{\scriptsize 0.000}&0.000$\pm$\textit{\scriptsize 0.000}\\
        \textbf{\textsc{Multi-Agent Debate}} & 0.320$\pm$\textit{\scriptsize 0.098} &0.000$\pm$\textit{\scriptsize 0.000} & 0.000$\pm$\textit{\scriptsize 0.000}&0.333$\pm$\textit{\scriptsize 0.211}&0.175$\pm$\textit{\scriptsize 0.061} &0.047$\pm$\textit{\scriptsize 0.016}&0.266$\pm$\textit{\scriptsize 0.133} &0.000$\pm$\textit{\scriptsize 0.000}&0.000$\pm$\textit{\scriptsize 0.000}\\
        \midrule
        \textbf{\textsc{TLVD}} &\textbf{0.800$\pm$\textit{\scriptsize 0.126}}&\textbf{0.886$\pm$\textit{\scriptsize 0.057}} &\textbf{0.467$\pm$\textit{\scriptsize 0.112}}&\textbf{0.734$\pm$\textit{\scriptsize 0.133}}&\textbf{0.914$\pm$\textit{\scriptsize 0.070}}&\textbf{0.350$\pm$\textit{\scriptsize 0.064}}&\textbf{0.933$\pm$\textit{\scriptsize 0.133}}&\textbf{0.900$\pm$\textit{\scriptsize 0.082}}&\textbf{0.412$\pm$\textit{\scriptsize 0.043}}\\
    \bottomrule
    \end{tabular}}
\end{table*}
\subsection{RQ3: Ablation Studies}
To evaluate the effectiveness of different components in TLVD, we design the following variants: \textbf{TLVD-v}: Removes the multi-agent LLM collaboration module and uses only a voting strategy. \textbf{TLVD-d}: Removes the multi-agent LLM collaboration module and uses only a debate strategy. \textbf{TLVD-I}: Uses a single LLM to generate latent variables and their semantics. \textbf{TLVD-R1}: Uses only the Action Likelihood Reward. \textbf{TLVD-R2}: Considers both the Action Likelihood Reward and the Uncertainty Reduction Reward. \textbf{TLVD-R3}: Considers the Action Likelihood Reward, the Uncertainty Reduction Reward, and the Collaborative Contribution Reward.
We conducted ablation experiments using WCHSU-Cancer (n=12) and WCHSU-Pain as examples. Additionally, we analyzed the impact of different LLM types on the results.

As shown in Figure \ref{fig:ablation} and Table \ref{table-llms} (see Appendix \ref{addition}), we have the following observations: (i) TLVD achieves the best performance when all components are included. (ii) The Uncertainty Reduction Reward and the Collaborative Contribution Reward have a significant impact on model performance. (iii) There are noticeable performance differences among different LLMs, with GPT-oss-120B and DeepSeek-v3 performing relatively better.
\subsection{RQ4: Parameter Sensitivity}
In the WCHSU-Cancer ($n=12$) and WCHSU-Pain datasets, we further investigate the impact of the number of execution LLMs $N$ and the discount factor $\gamma$ on model performance. As shown in Figure \ref{fig-parameter1}, the results indicate that when the number of execution LLMs exceeds six, performance gains become very limited, and in some cases, performance even declines. When using weaker LLMs (e.g., LLaMA 3.1 8B), the number of execution LLMs has little impact on model performance. In contrast, with stronger LLMs (e.g., LLaMA 3.1 70B), increasing the number of execution LLMs yields some performance improvements. We attribute this to the challenges faced by the coordinator LLM when managing too many execution LLMs, making it difficult to achieve effective coordination through information from additional agents. Moreover, as the discount factor $\gamma$ increases, model performance shows an upward trend.
\begin{table}[t]
    \centering
    \caption{Performance of different configurations.}
    \label{table_different_llms}
    \resizebox{0.48\textwidth}{!}{
    \newcommand{\tabincell}[2]{\begin{tabular}{@{}#1@{}}#2\end{tabular}}
    \begin{tabular}{l||ccc|ccc}
     \toprule    
      &\multicolumn{3}{c|}{\bf WCHSU-Cancer ($n=12$)} &\multicolumn{3}{c}{\bf WCHSU-Pain}\\
      \textbf{Methods} &\textbf{ACC}&\textbf{CAcc}&\textbf{ECit}&\textbf{ACC} &\textbf{CAcc}&\textbf{ECit}\\
     \midrule
     \midrule
        \textbf{Homo. (2$\times$ LLaMA3.1 8B)} &0.500&1.000 &0.115&0.500&1.000&0.474\\
        \textbf{Homo. (2$\times$ LLaMA3.1 70B)} &0.833&1.000&0.462&0.850&1.000&0.737 \\
        \textbf{Hetero. (LLaMA3.1 8B, Qwen2.5 7B)} &0.333&1.000 &0.346&0.250&1.000 &0.421\\
        \textbf{Hetero. (LLaMA3.1 70B, Qwen2.5 72B)} &0.500&1.000&0.346 &0.750&1.000 &0.632\\
        \midrule
        \textbf{TLVD} &0.833 &0.900&0.415&0.800 &0.886&0.453\\
        \textbf{LLAMA-3.1-8B (Zero-shot)} &0.000 &0.000&0.000&0.250&0.191&0.057\\
    \bottomrule
    \end{tabular}}
\end{table}
\subsection{RQ5: Case Studies}
We take the specific latent variable $L_1$ from the WCHSU-Cancer ($n$=22) dataset as an example to demonstrate the latent variables identified by our method and the supporting evidence. As shown in Figure \ref{fig:casestudy}, Executors 1 and 2 output $L_1$ and its semantics based on their own beliefs and observations. Through the latent variable verification process, retrieved information is returned. The coordinator aggregates all information to infer the final latent variables and simultaneously outputs the corresponding causal evidence.
\begin{figure}[t] 
\includegraphics[width=0.48\textwidth]{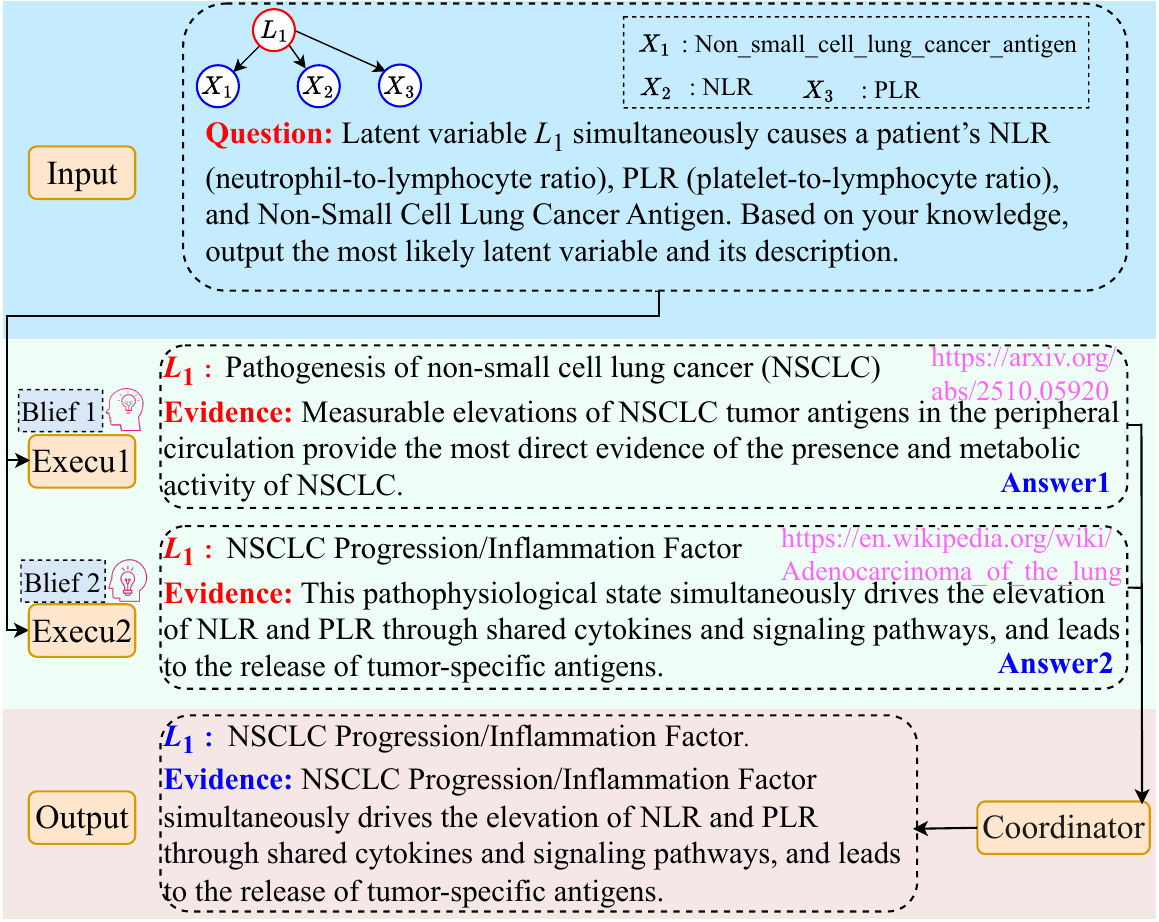}
  \centering
  \caption{Case study.}
  \label{fig:casestudy}
  \Description{figure-failure}
\end{figure}
\subsection{RQ6: Other In-Depth Analysis}
\subsubsection{Impact of Data Sources}
We further analyzed the impact of different data sources on TLVD to validate its applicability under varying conditions. We explored three scenarios on the WCHSU-Pain dataset: the W/O ARR model excludes literature from arXiv, the W/O WIKI model excludes literature from WIKIPEDIA, and the W/O DB model excludes medical-related knowledge graphs. As shown in Figure \ref{fig:data_source}, we observe that articles on arXiv contain more medically relevant causal information, which aligns well with reality and is consistent with existing studies \cite{chen2024rarebench}. The influence of the databases on our model is relatively small, which is reasonable since the data we use contains knowledge that is not publicly available and is rarely stored or used in existing databases.
\begin{figure}[t] 
\includegraphics[width=0.47\textwidth]{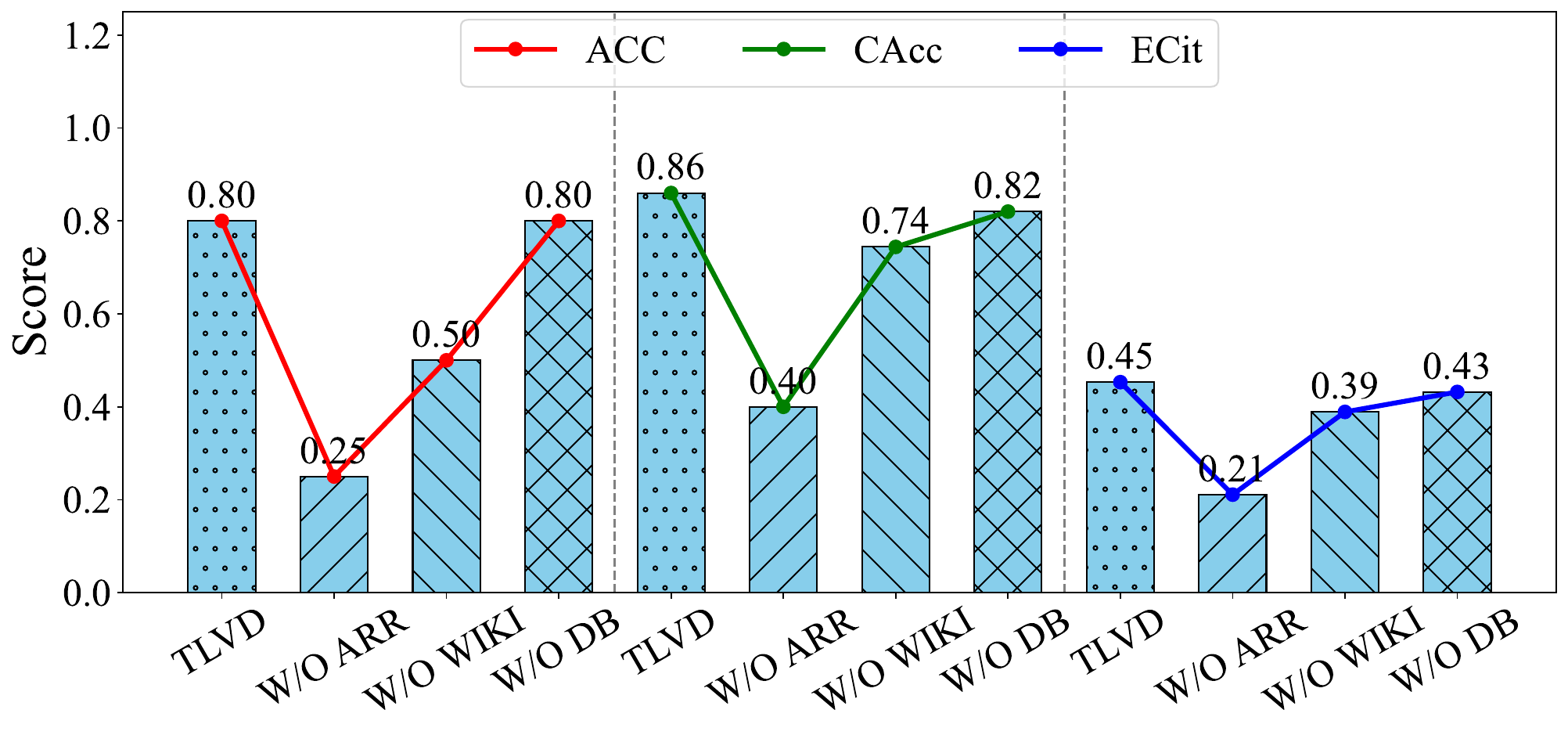}
  \centering
  \caption{Impact of data sources on model performance.}
  \label{fig:data_source}
  \Description{figure-failure}
\end{figure}
\begin{figure}[t] 
\includegraphics[width=0.48\textwidth]{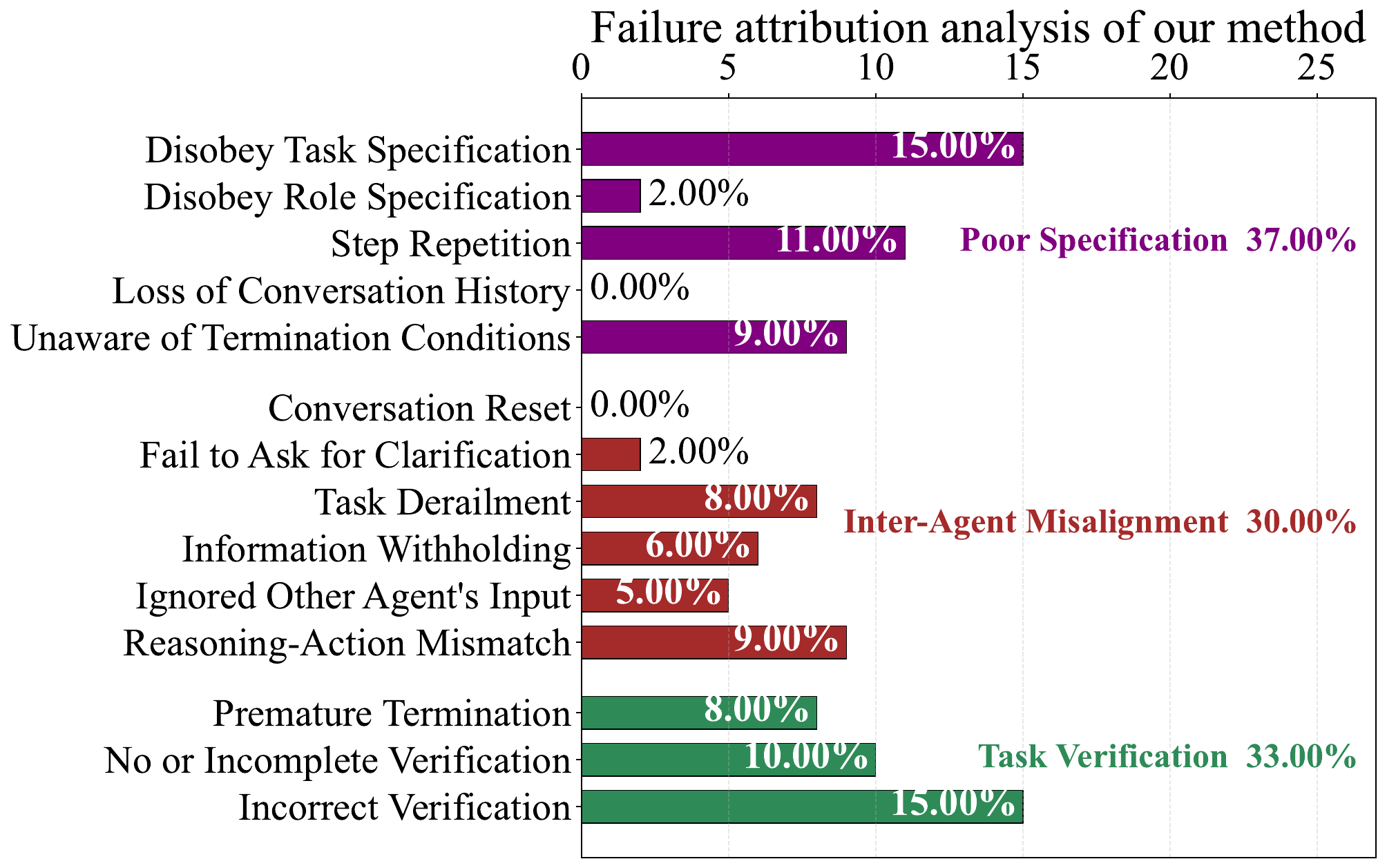}
  \centering
  \caption{Failure Attribution of Our Method. Note that the vertical axis represents specific failure modes, and the horizontal axis represents their respective proportions.}
  \label{fig:failure}
\end{figure}
\subsubsection{Error Analysis}
To further dissect our approach, we employed the MAST taxonomy \cite{cemri2025multi} to attribute failures in multi-agent systems. Through the analysis of 100 execution traces, we computed the frequency of occurrence for each failure mode and category. Detailed definitions and examples for each failure mode can be found in existing studies \cite{cemri2025multi}. As shown in Figure \ref{fig:failure}, the primary causes of multi-agent system failures are Specification Issues and Task Verification. For Specification Issues, failures stem from system design decisions and incomplete or ambiguous prompts, mainly including vague task instructions and unclear role definitions. For Task Verification, failures are primarily related to output quality control, such as insufficient verification or premature task termination.
\section{Conclusion}\label{sec:conclusion}
In this work, we propose TLVD, a multi-LLM collaborative framework for traceable latent variable discovery. TLVD first constructs a hierarchical coordination mechanism for multi-agent LLMs using game theory and reinforcement learning, enabling multiple execution LLMs to perform distributed reasoning under the guidance of a coordinator LLM to infer latent variables and their semantic meanings. Subsequently, the framework conducts causal validation of the inferred latent variables using multiple real-world data sources. Finally, we conduct extensive experiments on three real hospital datasets that we constructed, as well as two benchmark datasets, demonstrating the effectiveness of the TLVD in latent variable discovery.
\section{Acknowledgements}
This research is partially supported by funding from Xiangjiang Laboratory (25XJ02002), the National Natural Science Foundation of China (62376228, 62376227), the Science and Technology Innovation Program of Hunan Province (2024RC4008), the China Postdoctoral Science Foundation (2025M770766), Sichuan Provincial Postdoctoral Research Project Special Funding (TB2025043) and Sichuan Science and Technology Program (2023NSFSC0032).  Carl
Yang is not supported by any funds from China.

\bibliographystyle{ACM-Reference-Format}
\bibliography{sample-base}

\appendix
\section{Reproducibility}\label{data}
In this section, we provide more details about the datasets and evaluation metrics
to facilitate the reproducibility of the results. Our code is available at \url{https://github.com/HYJ9999/TLVD.git}.

\textbf{WCHSU-Cancer}: The WCHSU-Cancer dataset comes from the health management center of a large hospital in Asia, focusing on early screening for lung cancer, and includes health check-up data from a total of 200,000 participants. The original dataset contains 230 variables. WCHSU-Cancer (n=22) was curated by scoring the correlation between the original variables and the lung cancer variable of interest using LLMs, with scores ranging from 0 to 5. After filtering based on scores greater than 5, 22 relevant variables were selected. WCHSU-Cancer (n=12), on the other hand, was screened directly by lung cancer specialists from the same hospital.

\textbf{WCHSU-Pain}: The WCHSU-Pain comes from the same large hospital in Asia and focuses on postoperative pain in patients. It includes perioperative and postoperative analgesic usage data from a total of 1,568 patients. The dataset contains 16 variables. Please noted that all datasets used were properly pre-processed (including the de-identification of all patient information) and IRB-approved.

Next, we present the detailed computation of the evaluation metrics:
(1) ACC represents the proportion of latent variables correctly predicted by the model. For the WCHSU dataset which lacks ground-truth, we invited five experts from relevant fields at Asia's largest hospital to perform a consistency evaluation of the latent variables inferred by the model. Specifically, five experts assessed the latent variables and the authenticity of the evidence based on their professional experience; the consensus reached through consultation was adopted as the ground truth for this experiment.
(2) $CAcc=\frac{n}{EG}$, $n$ is the number of edges with real evidence, and $EG$ is the number of edges for which evidence is found. (3) $ECit=\frac{n}{EA}$, $EA$ is the number of edges in the causal graph where at least one of the nodes is a latent variable.


\section{Theoretical Proof}

\subsection{Proof of Theorem \ref{the1}
}\label{app-them1}
\begin{proof}
 We aim to prove the existence of a Bayesian Nash Equilibrium (BNE) in our multi-agent LLM framework under the specified conditions. Following existing studies, our proof primarily proceeds by verifying the conditions of Glicksberg’s Fixed Point Theorem, which guarantees the existence of a fixed point in continuous games with finite-dimensional strategy spaces.

\textbf{Step 1: Define the Best Response Correspondence} For each agent $i$, define the best response correspondence $BR_i$ as:
\begin{equation*}
    BR_i(\pi_{-i}) = \{\pi_i \in \Pi_i \mid \underset{\pi_i}{\text{ max}}\; U_i(\theta_i, \pi_i, \pi_{-i})\},
\end{equation*}
where $\Pi_i$ is the set of all admissible strategies for agent $i$, and $\pi_{-i}$ denotes the strategies of all other agents.

\textbf{Step 2: Verify the Conditions of Glicksberg’s Fixed Point Theorem} To apply Glicksberg’s theorem, we need to verify the following conditions for each agent $i$:

\textbf{(i) Compactness and Convexity of Strategy Space}: The strategy space $\Pi_i$ consists of all measurable functions from the type space $\Theta_i$ to the action space $A_i$. By Tychonoff's theorem, since $\Theta_i$ and $A_i$ are compact metric spaces, $\Pi_i$ is also compact. Convexity follows similarly. \textbf{(ii) Continuity of the Payoff Function}: For fixed $\theta_i$, since $U_i$ is continuous in actions and strategies map continuously from types to actions, the composition\\ $U_i(\theta_i, \pi_i(\theta_i), \pi_{-i}(\theta_{-i}))$ is continuous in $(\pi_i, \pi_{-i})$. \textbf{(III) Quasi-Conc\\avity of the Payoff Function}: For each $\theta_i$ and $\pi_{-i}$, the mapping $\pi_i \mapsto U_i(\theta_i, \pi_i, \pi_{-i})$ is linear in the space of mixed strategies, ensuring quasi-concavity.

\textbf{Step 3: Establish Upper Hemicontinuity and Non-Empty, Convex-Valuedness of the Best Response Correspondence} We need to show that $BR_{i}\left ( \pi _{-i} \right ) $ is upper hemicontinuous with non-empty, convex values.

\textbf{(i) Non-Empty and Convex Values}:
    The payoff function $U_i$ is continuous and quasi-concave in $\pi_i$. By the Weierstrass Theorem, the extreme value must exist (and thus be non-empty); convexity follows from the quasi-concavity of $U_i$.
\textbf{(ii) Upper Hemicontinuity}
According to Berge’s Maximum Theorem, if: (1)The constraint set $\Pi_i$ is compact. (2) The objective function $U_i$ is continuous in $(\pi_i, \pi_{-i})$, then the best response correspondence $BR_i(\pi_{-i})$ is guaranteed to be upper hemicontinuous.

\textbf{Step 4: Application of Glicksberg’s Fixed Point Theorem}
After verifying the above conditions, applying Glicksberg’s Fixed Point Theorem in our framework leads to the following conclusion: if each player’s strategy set is compact and convex, and their payoff functions are continuous and quasi-concave in their own strategies, then the game has at least one mixed-strategy Nash equilibrium. That is, there exists a strategy profile $\pi^* = (\pi^*_1, \pi^*_2, \dots, \pi^*_N)$ such that for each agent $i$:
$\pi^*_i \in BR_i(\pi^*_{-i})$,
meaning that, given their beliefs about the types and strategies of other agents, no agent has an incentive to unilaterally deviate to improve their expected payoff. This strategy profile constitutes a Bayesian Nash Equilibrium in our TVLD framework.

\end{proof}
\subsection{Proof of Convergence to BNE
}\label{convergence}
\begin{proof}
We aim to show that by minimizing the TD loss of each agent’s Q-network, the agents’ strategies converge to a BNE.

\textbf{Core Assumptions}:

1. The Q-network $Q_i(s, a_i; \theta_i)$ is parameterized by prompt embeddings $\theta_i$, and the mapping from $\theta_i$ to $Q_i$ is continuously differentiable.
2. The exploration strategy ensures sufficient coverage of the state–action space.
3. The loss function $L_i(\theta_i)$ is convex, or its gradient is Lipschitz continuous with respect to $\theta_i$.
4. The gradient $\nabla_{\theta_i} L_i(\theta_i)$ is Lipschitz continuous.
5. The learning rate $\eta_t$ satisfies the Robbins–Monro conditions: $\sum_{t=1}^\infty \eta_t = \infty$ and $\sum_{t=1}^\infty \eta_t^2 < \infty$.

\textbf{Step 1: Defining the TD Loss Function}
For agent $i$, the TD loss is:
\begin{equation*}
    L_i(\theta_i) = \mathbb{E}_{(s,a_i,r_i,s') \sim D_i} \Big[ \big(r_i + \gamma \max_{a'_i} Q_i(s', a'_i; \theta_i^-) - Q_i(s, a_i; \theta_i)\big)^2 \Big],
\end{equation*}
which measures the discrepancy between the predicted Q-value and the target Q-value based on the reward and estimated optimal future Q-value.

\textbf{Step 2: Convergence of Gradient Descent with TD Loss}
Under Assumptions 1--4, stochastic gradient descent converges almost surely to a stationary point $\theta_i^*$ of $L_i(\theta_i)$:
\[
\lim_{t \to \infty} \theta_i^{(t)} = \theta_i^* \quad \text{and} \quad \nabla_{\theta_i} L_i(\theta_i^*) = 0.
\]
This follows from standard stochastic approximation theory. The gradient condition implies:
\begin{equation*}
\scalebox{0.96}{$
\begin{aligned}
    \mathbb{E}_{(s,a_i,r_i,s') \sim D_i} \Big[ 
        &\big(r_i + \gamma \max_{a'_i} Q_i(s', a'_i; \theta_i^-) 
        - Q_i(s, a_i; \theta_i^*)\big) \\
        &\cdot \nabla_{\theta_i} Q_i(s, a_i; \theta_i^*) 
    \Big] = 0
\end{aligned}
$}
\end{equation*}
If the function approximator is sufficiently expressive, this implies that the TD error vanishes in expectation:
\[
Q_i(s, a_i; \theta_i^*) = \mathbb{E} \left[ r_i + \gamma \max_{a'_i} Q_i(s', a'_i; \theta_i^*) \mid s, a_i \right].
\]
Thus, $Q_i(\cdot; \theta_i^*)$ satisfies the Bellman optimality equation. This ensures that the agent’s policy $\pi_i(a_i | s; \theta_i^*)$ is a best response to the current policies of the other agents, as it maximizes expected cumulative rewards.

\textbf{Step 3: Establishing Bayesian Nash Equilibrium}
Since each agent’s policy is a best response to others, the set of policies ${\pi_i^*}$ constitutes a BNE. At this equilibrium, each agent maximizes its expected utility given its beliefs about other agents’ types and strategies, thus fulfilling the definition of BNE.
\end{proof}
\section{Detailed Proofs}
\subsection{Proof of Lemma \ref{them2}}\label{proof-them2}
\begin{proof}
    Consider the value functions under policies $\pi'$ and $\pi$:
\begin{equation*}
\left\{
\begin{aligned}
V_i^{\pi'}(s) &= \mathbb{E}_{\pi'}\left[ \sum_{k=0}^{\infty} \gamma^k r_i(s_k, a_k)\mid s_0 = s \right], \\
V_i^\pi(s) &= \mathbb{E}_\pi\left[ \sum_{k=0}^{\infty} \gamma^k r_i(s_k, a_k)\mid s_0 = s \right].
\end{aligned}
\right.
\end{equation*}
Subtracting the two equations yields:
\begin{equation*}
\begin{aligned}
    V_{\pi'}(s) - V_\pi(s) &= \mathbb{E}_{\pi'}\left[ \sum_{k=0}^{\infty} \gamma^k r_i(s_k, a_k) \right] - \mathbb{E}_\pi\left[ \sum_{k=0}^{\infty} \gamma^k r_i(s_k, a_k) \right]\\
&= \sum_{k=0}^{\infty} \gamma^k \left( \mathbb{E}_{s_k \sim d_k^{\pi'}}[r_i(s_k, a_k)] - \mathbb{E}_{s_k \sim d_k^\pi}[r_i(s_k, a_k)] \right).
\end{aligned}
\end{equation*}
Assuming the difference in state distributions is negligible, we focus on action differences. Using the Q-function definition:
\begin{equation*}
    Q_\pi(s, a_i, a_{-i}) = r_i(s, a_i, a_{-i}) + \gamma \mathbb{E}_{s' \sim P}[V_\pi(s')],
\end{equation*}
we can write:
\begin{equation*}
    V_{\pi'}(s) - V_\pi(s) = \sum_{k=0}^{\infty} \gamma^k \mathbb{E}_{s_k \sim d_k^{\pi'}}\left[ Q_\pi(s_k, a'_k) - V_\pi(s_k) \right].
\end{equation*}
Since $V_\pi(s_k) = \mathbb{E}_{a_k \sim \pi(s_k)}[Q_\pi(s_k, a_k)]$, we have:
\begin{equation*}
\begin{aligned}
\begin{aligned}
V_{\pi'}(s) - V_\pi(s) &= \sum_{k=0}^{\infty} \gamma^k \mathbb{E}_{s_k \sim d_k^{\pi'}} \left[ \mathbb{E}_{a'_k \sim \pi'(s_k)} Q_\pi(s_k, a'_k) \right. \\
&\quad - \left. \mathbb{E}_{a_k \sim \pi(s_k)} Q_\pi(s_k, a_k) \right].
\end{aligned}
\end{aligned}
\end{equation*}
Switching the order of expectations and summing over $k$, we get:
\begin{equation*}
    V_{\pi'}(s) - V_\pi(s) = \frac{1}{1 - \gamma} \mathbb{E}_{s \sim d_{\pi'}} \left[ Q_\pi(s, a'_i, a'_{-i}) - Q_\pi(s, a_i, a_{-i}) \right].
\end{equation*}
\end{proof}
\subsection{Bounding the Bayesian Regret}\label{regret}
To prove the regret bound in our framework is
$R(T) = O\left( \frac{N \sqrt{T}}{1 - \gamma} \right)$,
we employ the given lemma:
\begin{equation*}
\scalebox{0.90}{$
    V_i^*(s_t) - V_i^{\pi_t}(s_t) = \frac{1}{1-\gamma} \mathbb{E}_{a_i^{*t}, a_{-i}^{*t}, a_i^t, a_{-i}^t} \left[ Q_i^{\pi_t}(s_t, a_i^{*t}, a_{-i}^{*t}) - Q_i^{\pi_t}(s_t, a_i^t, a_{-i}^t) \right]$},
\end{equation*}
and combine it with the convergence properties of policy gradient methods.
The detailed derivation is as follows:
\subsubsection{Advantage Function Bound}
From the convergence theory of policy gradient methods (based on online convex optimization) \cite{hazan2016introduction}, for each agent $i$, state $s_t$ and optimal action $a_i^{*t}$, the expected advantage function satisfies 
\begin{equation}\label{eq1}
    \mathbb{E} \left[ A_i^{\pi_t}(s_t, a_i^{*t}) \right] \le \frac{C}{\sqrt{t}}\ ,
\end{equation}
where $A_i^{\pi_t}(s_t, a_i^{*t}) = Q_i^{\pi_t}(s_t, a_i^{*t}) - V_i^{\pi_t}(s_t) )$,
and $C=O\left ( \frac{R_{\max}}{1-\gamma}  \right ) > 0$ is a constant depending on the reward bound $R_{\max}=1$ and the geometry of the policy space.

\subsubsection{Application of the Lemma}
According to the lemma above, the value function gap is linked to the expected advantage function:
\begin{equation}\label{eq2}
    V_i^*(s_t) - V_i^{\pi_t}(s_t) = \frac{1}{1-\gamma} \mathbb{E} \left[ A_i^{\pi_t}(s_t, a_i^{*t}) \right].
\end{equation}
This follows from the linearity of expectation and the determinism of the optimal action (i.e., when $a_i^{*t}$ is fixed, $\mathbb{E}_{a_i^{*t}}[A_i^{\pi_t}(s_t, a_i^{*t})] = A_i^{\pi_t}(s_t, a_i^{*t})$).

\subsubsection{Bound on Value Function Difference}
Combining Equation \ref{eq1} and Equation \ref{eq2}, we obtain
$V_i^*(s_t) - V_i^{\pi_t}(s_t) \le \frac{1}{1-\gamma} \cdot \frac{C}{\sqrt{t}}$.

\subsubsection{Individual Agent Regret Summation}
The cumulative regret of agent $i$ over $T$ steps is
\begin{equation*}
    R_i(T) = \mathbb{E} \left[ \sum_{t=1}^T \left( V_i^*(s_t) - V_i^{\pi_t}(s_t) \right) \right]
\le \frac{C}{1-\gamma} \sum_{t=1}^T \frac{1}{\sqrt{t}}.
\end{equation*}

\subsubsection{Bounding the Harmonic Sum}
Using a standard integral bound on the partial harmonic series with exponent $1/2$: $
\sum_{t=1}^T \frac{1}{\sqrt{t}} \le 2 \sqrt{T}$.

\subsubsection{Global Regret Bound}

The total regret is the sum over all agents:
\begin{equation*}
    R(T) = \sum_{i=1}^N R_i(T) \le \sum_{i=1}^N \frac{2C \sqrt{T}}{1-\gamma} = \frac{2 N C \sqrt{T}}{1 - \gamma} = O\left( \frac{N \sqrt{T}}{1 - \gamma} \right).
\end{equation*}
Note that this proof relies on the convergence property of policy gradient methods, which ensures that the advantage function $A_i^{\pi_t}(s_t, a^{*t})$ decreases at a rate of $O(1/\sqrt{t})$. This convergence rate can be derived from online convex optimization theory, such as mirror descent.

\section{More experimental results}\label{addition}
We conducted additional experiments on two benchmark datasets with ground-truth labels: Multitasking Behaviour and Teacher’s Burnout Study. As shown in Table \ref{table-bench}, our method still achieves the best performance. Although these two datasets might have been included in the training data of LLMs, all the baselines we compared against are also LLM-based, making the comparison reasonably fair and still demonstrating the effectiveness of our method.

Additionally, to further investigate the impact of different causal discovery algorithms on our framework, we employed Hier.rank \cite{huang2022latent} to uncover latent causal graph structures. As shown in Table \ref{table-hier}, the results indicate that our method is robust and maintains optimal performance across different causal discovery algorithms. Although some edges still lack explicit causal evidence, based on the identifiability guarantees of the Hier.rank and RLCD \cite{dong2024versatile} algorithms, the latent variables inferred by our framework are also highly likely to be correct. The results of the ablation experiments are shown in Table \ref{table-llms} and Figure \ref{fig:ablation}. See Figure \ref{fig-parameter1} for the parameter analysis.

Finally, we also analyzed the token consumption of our method across different modules and sample sizes. As shown in Figure \ref{cost-parameter}, the token consumption is higher during the latent variable validation phase, primarily due to the processing of papers and Wiki texts containing rich textual content. In contrast, the latent variable proposal phase requires relatively few tokens, mainly thanks to modeling agent interactions via the belief network. Although our method consumes slightly more tokens than the baseline models, this is acceptable given the performance improvements achieved. Moreover, the token usage of our method is primarily determined by the number of latent variables and the size of their corresponding Markov blankets.

\begin{table}[t]
    \caption{Performance comparison on benchmark datasets.}
    \label{table-bench}
    \centering
    \resizebox{0.48\textwidth}{!}{
    \newcommand{\tabincell}[2]{\begin{tabular}{@{}#1@{}}#2\end{tabular}}
    \centering
    \resizebox{\textwidth}{!}{
    \begin{tabular}{l||ccc|ccc}
     \toprule    
      &\multicolumn{3}{c|}{\bf Multitasking Behaviour} &\multicolumn{3}{c}{\bf Teacher’s Burnout Study}\\
      \textbf{Methods} &\textbf{ACC}&\textbf{CAcc}&\textbf{ECit}&\textbf{ACC}&\textbf{CAcc}&\textbf{ECit}\\
     \midrule
     \midrule
        \textbf{GPT5} &0.250$\pm$\textit{\scriptsize 0.047} &0.083$\pm$\textit{\scriptsize 0.053} &0.083$\pm$\textit{\scriptsize 0.053}&0.145$\pm$\textit{\scriptsize 0.045}&0.069$\pm$\textit{\scriptsize 0.029}&0.032$\pm$\textit{\scriptsize 0.013}\\
    \midrule
        \textbf{Gemini-deepresearch}  &\textbf{0.500$\pm$\textit{\scriptsize 0.000}} &0.657$\pm$\textit{\scriptsize 0.070} &\textbf{0.383$\pm$\textit{\scriptsize 0.041}} &0.364$\pm$\textit{\scriptsize 0.100} &0.315$\pm$\textit{\scriptsize 0.038}&0.146$\pm$\textit{\scriptsize 0.017}\\
        \textbf{OpenAI-deepresearch} &0.350$\pm$\textit{\scriptsize 0.122}&0.150$\pm$\textit{\scriptsize 0.062}&0.150$\pm$\textit{\scriptsize 0.062}&0.564$\pm$\textit{\scriptsize 0.068}&0.762$\pm$\textit{\scriptsize 0.049}&0.354$\pm$\textit{\scriptsize 0.023}\\
        \textbf{Qwen-deepresearch} &0.050$\pm$\textit{\scriptsize 0.100}&0.000$\pm$\textit{\scriptsize 0.000}&0.000$\pm$\textit{\scriptsize 0.000}&0.164$\pm$\textit{\scriptsize 0.068}&0.154$\pm$\textit{\scriptsize 0.024}&0.071$\pm$\textit{\scriptsize 0.011}\\
        \textbf{Doubao-deepresearch} &0.000$\pm$\textit{\scriptsize 0.000} &0.000$\pm$\textit{\scriptsize 0.000}&0.000$\pm$\textit{\scriptsize 0.000}&0.127$\pm$\textit{\scriptsize 0.045}&0.131$\pm$\textit{\scriptsize 0.039} &0.061$\pm$\textit{\scriptsize 0.018}\\
        \midrule
        \textbf{Autogen} & 0.050$\pm$\textit{\scriptsize 0.100}&0.000$\pm$\textit{\scriptsize 0.000} &0.000$\pm$\textit{\scriptsize 0.000} &0.255$\pm$\textit{\scriptsize 0.068} &0.208$\pm$\textit{\scriptsize 0.039} &0.096$\pm$\textit{\scriptsize 0.018} \\
\textbf{MiniMax}&0.150$\pm$\textit{\scriptsize 0.122} &0.000$\pm$\textit{\scriptsize 0.000} &0.000$\pm$\textit{\scriptsize 0.000} &0.327$\pm$\textit{\scriptsize 0.109} &0.300$\pm$\textit{\scriptsize 0.051} &0.139$\pm$\textit{\scriptsize 0.026} \\
    \midrule
                \textbf{\textsc{CAMEL}} & 0.050$\pm$\textit{\scriptsize 0.100}&0.000$\pm$\textit{\scriptsize 0.000}&0.000$\pm$\textit{\scriptsize 0.000}&0.236$\pm$\textit{\scriptsize 0.072}&0.192$\pm$\textit{\scriptsize 0.024} &0.089$\pm$\textit{\scriptsize 0.011} \\
                \textbf{\textsc{Multi-Agent-Majority}} & 0.100$\pm$\textit{\scriptsize 0.122} &0.167$\pm$\textit{\scriptsize 0.105}&0.083$\pm$\textit{\scriptsize 0.053}&0.273$\pm$\textit{\scriptsize 0.141}&0.185$\pm$\textit{\scriptsize 0.029} &0.086$\pm$\textit{\scriptsize 0.013} \\
            \textbf{\textsc{Multi-Agent Debate}} & 0.000$\pm$\textit{\scriptsize 0.000} &0.000$\pm$\textit{\scriptsize 0.000}&0.000$\pm$\textit{\scriptsize 0.000}&0.291$\pm$\textit{\scriptsize 0.134}&0.162$\pm$\textit{\scriptsize 0.015} &0.075$\pm$\textit{\scriptsize 0.007}\\
            \midrule
        \textbf{\textsc{TLVD}} &\textbf{0.500$\pm$\textit{\scriptsize 0.000}}&\textbf{0.820$\pm$\textit{\scriptsize 0.095}} &0.283$\pm$\textit{\scriptsize 0.041}&\textbf{0.582$\pm$\textit{\scriptsize 0.045}}&\textbf{0.769$\pm$\textit{\scriptsize 0.049}}&\textbf{0.357$\pm$\textit{\scriptsize 0.023}}\\
    \bottomrule
    \end{tabular}}}
\end{table}
\begin{table}[t]
    \centering
    \caption{Performance comparison of models using the Hier. rank  \cite{huang2022latent} method.}
    \label{table-hier}
    \resizebox{0.48\textwidth}{!}{
    \newcommand{\tabincell}[2]{\begin{tabular}{@{}#1@{}}#2\end{tabular}}
    \begin{tabular}{l||ccc|ccc}
     \toprule    
      &\multicolumn{3}{c|}{\bf WCHSU-Cancer ($n=12$)} &\multicolumn{3}{c}{\bf WCHSU-Pain}\\
      \textbf{Methods} &\textbf{ACC}&\textbf{CAcc}&\textbf{ECit}&\textbf{ACC} &\textbf{CAcc}&\textbf{ECit}\\
     \midrule
     \midrule
       \textbf{GPT5} &0.200$\pm$\textit{\scriptsize 0.245} &0.125$\pm$\textit{\scriptsize 0.065} &0.125$\pm$\textit{\scriptsize 0.065}&0.300$\pm$\textit{\scriptsize 0.187}&0.100$\pm$\textit{\scriptsize 0.033}&0.100$\pm$\textit{\scriptsize 0.033}\\
    \midrule
        \textbf{Gemini-deepresearch}  &0.100$\pm$\textit{\scriptsize 0.200} &0.000$\pm$\textit{\scriptsize 0.000} &0.000$\pm$\textit{\scriptsize 0.000} &0.100$\pm$\textit{\scriptsize 0.122} &0.191$\pm$\textit{\scriptsize 0.021}&0.163$\pm$\textit{\scriptsize 0.018}\\
        \textbf{OpenAI-deepresearch} & 0.500$\pm$\textit{\scriptsize 0.316}&0.250$\pm$\textit{\scriptsize 0.258}&0.250$\pm$\textit{\scriptsize 0.135}&0.450$\pm$\textit{\scriptsize 0.100}&0.450$\pm$\textit{\scriptsize 0.041} &0.200$\pm$\textit{\scriptsize 0.018}\\
        \textbf{Qwen-deepresearch} &0.200$\pm$\textit{\scriptsize 0.245}&0.217$\pm$\textit{\scriptsize 0.027}&0.208$\pm$\textit{\scriptsize 0.000}&0.300$\pm$\textit{\scriptsize 0.100}&0.443$\pm$\textit{\scriptsize 0.029}&0.237$\pm$\textit{\scriptsize 0.018}\\
        \textbf{Doubao-deepresearch} &0.100$\pm$\textit{\scriptsize 0.200} &0.208$\pm$\textit{\scriptsize 0.037}&0.250$\pm$\textit{\scriptsize 0.113}&0.150$\pm$\textit{\scriptsize 0.122}&0.000$\pm$\textit{\scriptsize 0.000} &0.000$\pm$\textit{\scriptsize 0.000}\\
        \midrule
        \textbf{Autogen} &0.200$\pm$\textit{\scriptsize 0.245} &0.083$\pm$\textit{\scriptsize 0.026}&0.083$\pm$\textit{\scriptsize 0.026}&0.250$\pm$\textit{\scriptsize 0.158} &0.000$\pm$\textit{\scriptsize 0.000} &0.000$\pm$\textit{\scriptsize 0.000} \\
\textbf{MiniMax}&0.300$\pm$\textit{\scriptsize 0.245} &0.216$\pm$\textit{\scriptsize 0.017} &0.216$\pm$\textit{\scriptsize 0.017} & 0.350$\pm$\textit{\scriptsize 0.122}&0.412$\pm$\textit{\scriptsize 0.037} &0.259$\pm$\textit{\scriptsize 0.023} \\
    \midrule
                \textbf{\textsc{CAMEL}} & 0.400$\pm$\textit{\scriptsize 0.200}&0.000$\pm$\textit{\scriptsize 0.000} &0.000$\pm$\textit{\scriptsize 0.000}&0.350$\pm$\textit{\scriptsize 0.122}&0.000$\pm$\textit{\scriptsize 0.000} &0.000$\pm$\textit{\scriptsize 0.000}\\
                \textbf{\textsc{Multi-Agent-Majority}} & 0.300$\pm$\textit{\scriptsize 0.245} &0.000$\pm$\textit{\scriptsize 0.000}&0.000$\pm$\textit{\scriptsize 0.000}&0.300$\pm$\textit{\scriptsize 0.100}&0.000$\pm$\textit{\scriptsize 0.000}&0.000$\pm$\textit{\scriptsize 0.000} \\
            \textbf{\textsc{Multi-Agent Debate}} & 0.400$\pm$\textit{\scriptsize 0.200} &0.286$\pm$\textit{\scriptsize 0.090}&0.092$\pm$\textit{\scriptsize 0.017}&0.500$\pm$\textit{\scriptsize 0.000}&0.600$\pm$\textit{\scriptsize 0.122} &0.104$\pm$\textit{\scriptsize 0.028}\\
            \midrule
        \textbf{\textsc{TLVD}} &\textbf{0.700$\pm$\textit{\scriptsize 0.245}}&\textbf{1.000$\pm$\textit{\scriptsize 0.035}} &\textbf{0.646$\pm$\textit{\scriptsize 0.012}}&\textbf{0.750$\pm$\textit{\scriptsize 0.106}}&\textbf{1.000$\pm$\textit{\scriptsize 0.172}}&\textbf{0.471$\pm$\textit{\scriptsize 0.164}}\\
    \bottomrule
    \end{tabular}}
\end{table}
\begin{table}[t]
    \centering
    \caption{Comparison of performance using different LLMs.}
    \label{table-llms}
    \resizebox{0.44\textwidth}{!}{
    \newcommand{\tabincell}[2]{\begin{tabular}{@{}#1@{}}#2\end{tabular}}
    \begin{tabular}{l||ccc|ccc}
     \toprule    
      &\multicolumn{3}{c|}{\bf WCHSU-Cancer ($n=12$)} &\multicolumn{3}{c}{\bf WCHSU-Pain}\\
      \textbf{Methods} &\textbf{ACC}&\textbf{CAcc}&\textbf{ECit}&\textbf{ACC} &\textbf{CAcc}&\textbf{ECit}\\
     \midrule
     \midrule
        \textbf{Qwen3 235B} &0.833&1.000 &0.385&0.500&1.000&0.316\\
        \textbf{Llama 4 Maverick} &0.500&0.750 &0.115&0.750&1.000 &0.474\\
        \textbf{GPT-oss-120B} &0.667&1.000 &0.269 &1.000&1.000 &0.632\\
        \textbf{DeepSeek-v3} &0.833 &1.000&0.385&0.750 &0.857&0.368\\
    \bottomrule
    \end{tabular}}
\end{table}
\begin{figure}[t] 
\includegraphics[width=0.41\textwidth]{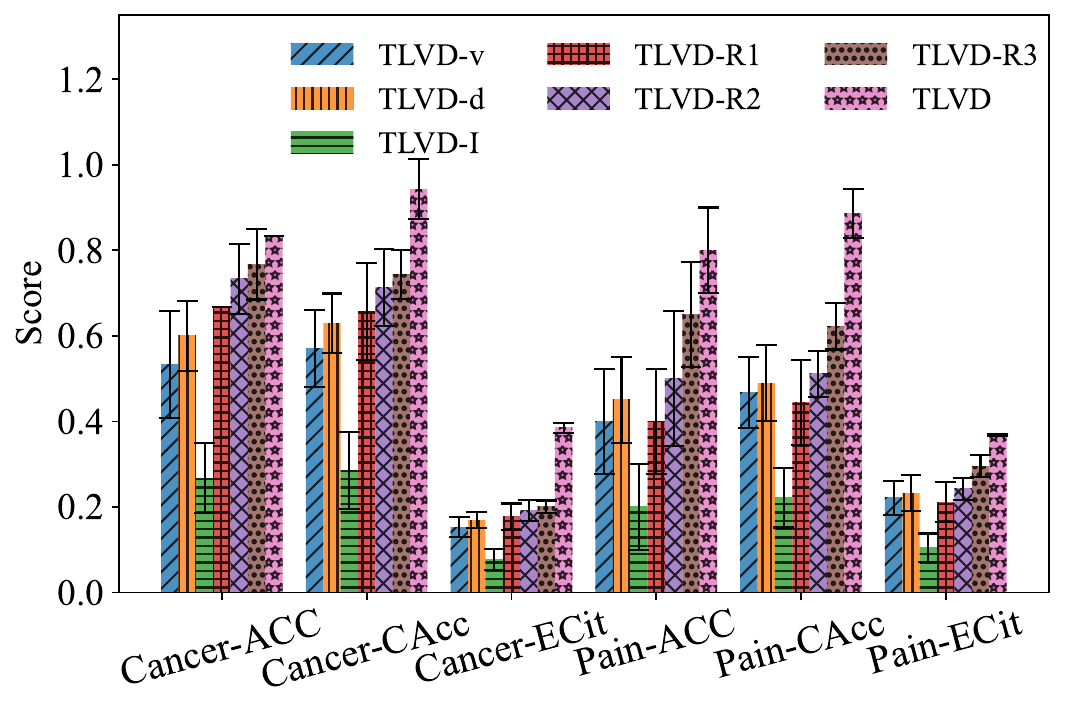}
  \centering
  \caption{Performance of evaluation for ablation study on WCHSU-Cancer (n=12) and WCHSU-Pain.}
  \label{fig:ablation}
  \Description{figure-failure}
\end{figure}
\begin{figure}[t]
	\centering
        \subfloat[WCHSU-Cancer] {\includegraphics[width=0.237\textwidth]{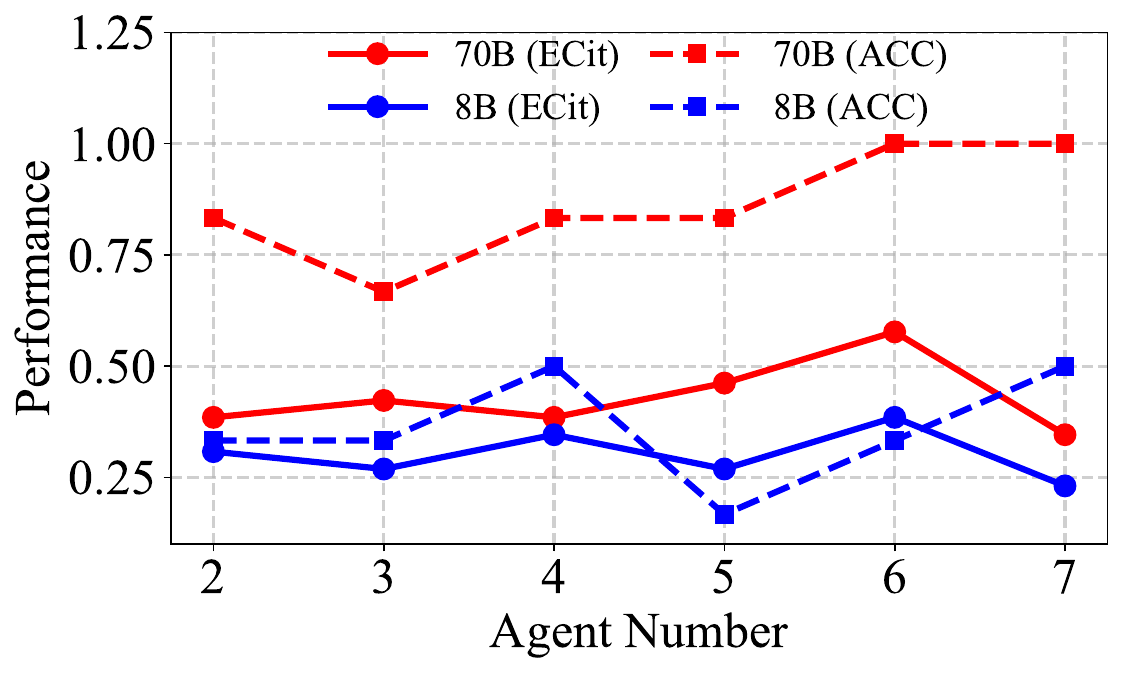}}
	\hfill
	\subfloat[WCHSU-Pain] {\includegraphics[width=0.237\textwidth]{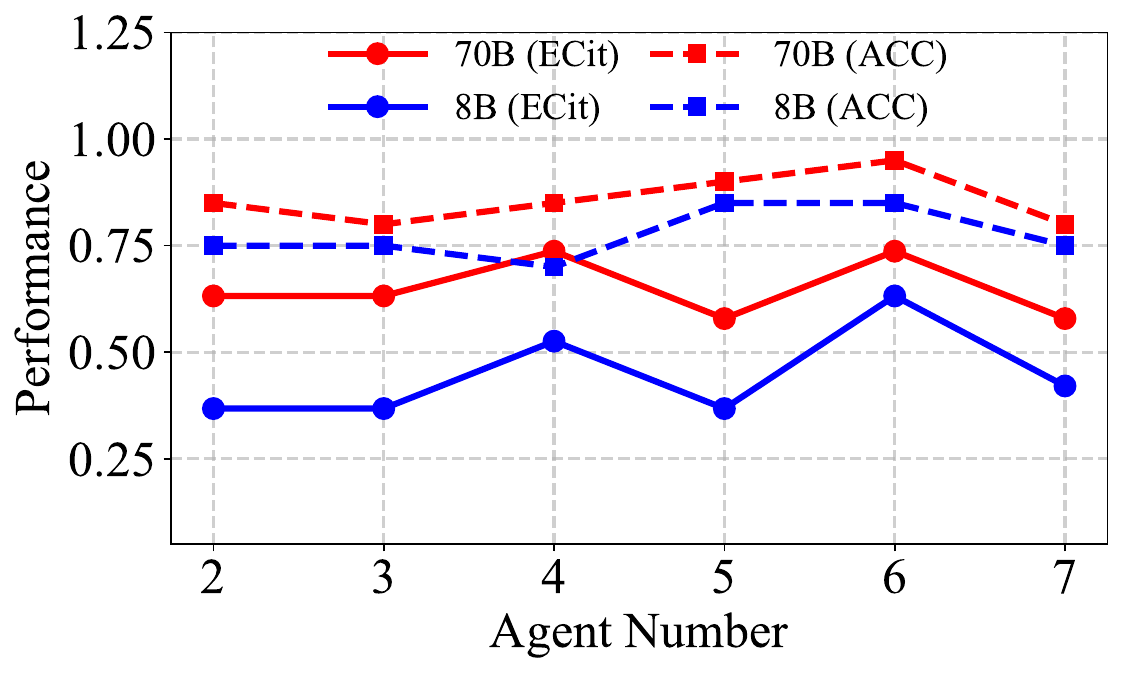}}
    \newline
    \subfloat[WCHSU-Cancer] {\includegraphics[width=0.237\textwidth]{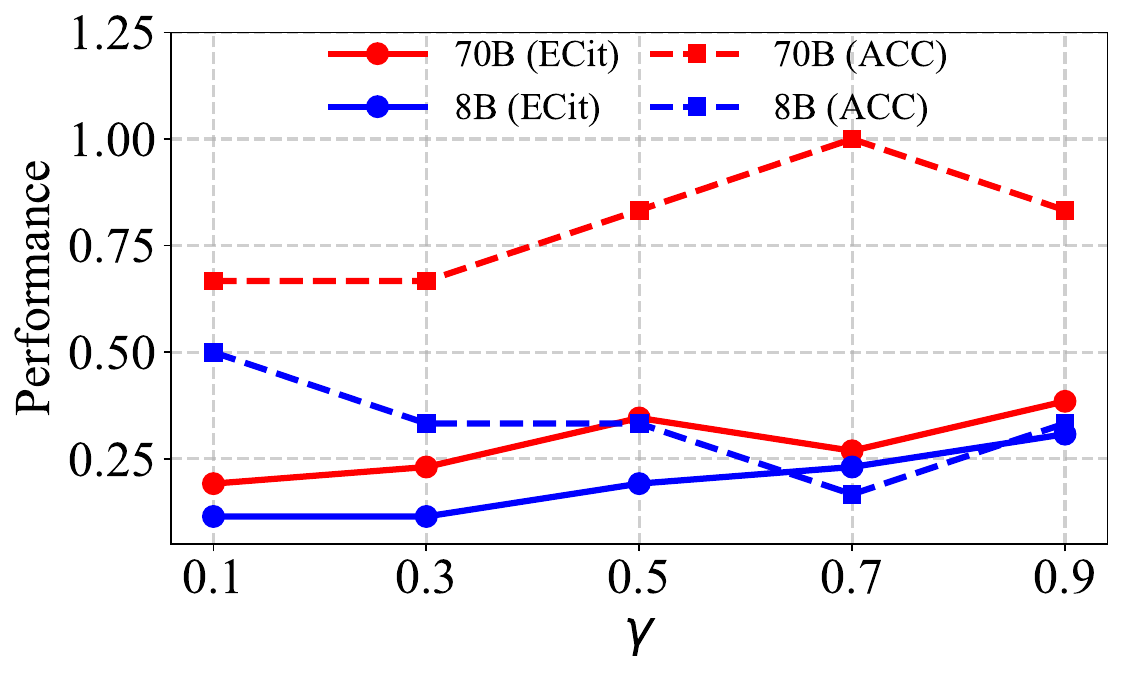}}
	\hfill
	\subfloat[WCHSU-Pain] {\includegraphics[width=0.237\textwidth]{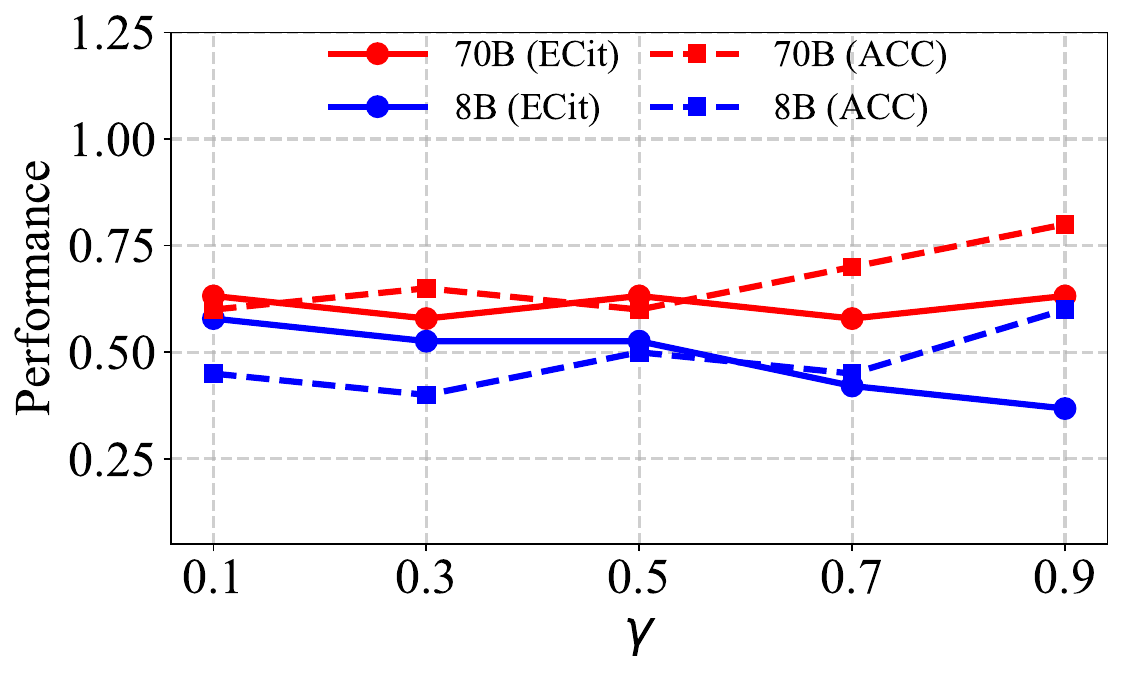}}
	\caption{Performance comparison with different $N$ and $\gamma$.}
 \Description{WCHSU-parameter}
	\label{fig-parameter1}
\end{figure}

\begin{figure}[t]
	\centering
            \subfloat[WCHSU-Cancer] {{\includegraphics[width=0.236\textwidth]{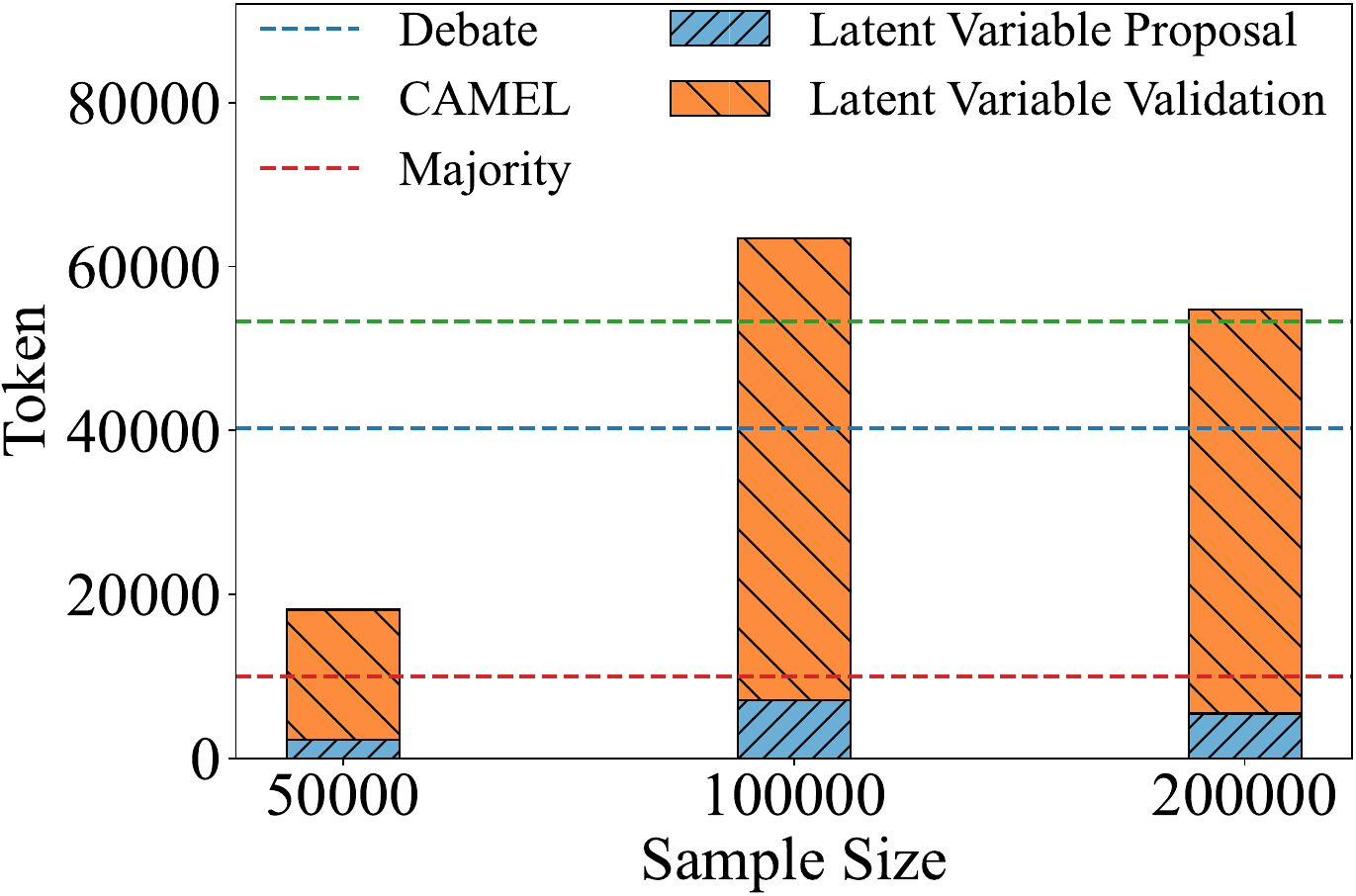}}}
	\hfill
	\subfloat[WCHSU-Pain] {\includegraphics[width=0.236\textwidth]{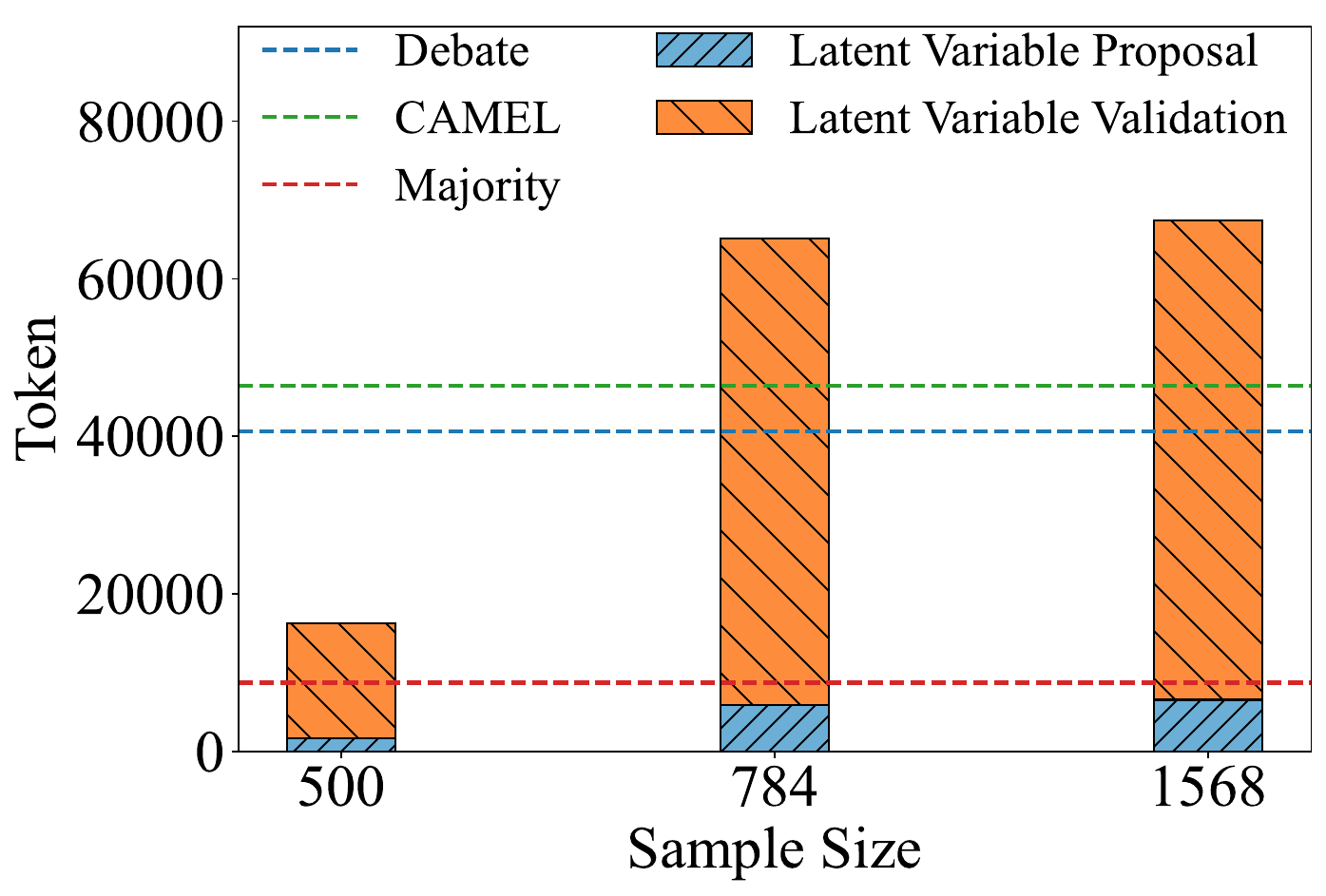}}
	\caption{The cost statistics of our model.}
    \Description{cost}
	\label{cost-parameter}
\end{figure}

\end{document}